\documentclass{article}

\usepackage{PRIMEarxiv}
\usepackage{amsthm}
\newtheorem{theorem}{Theorem}
\newtheorem{lemma}{Theorem}

\usepackage[utf8]{inputenc} 
\usepackage[T1]{fontenc}    
\usepackage{hyperref}       
\usepackage{url}            
\usepackage{booktabs}       
\usepackage{amsfonts}       
\usepackage{nicefrac}       
\usepackage{microtype}      
\usepackage{lipsum}
\usepackage{fancyhdr}       
\usepackage{graphicx}       
\graphicspath{{media/}}     

\usepackage{amsmath}
\usepackage{subcaption}
\usepackage{algorithm}
\usepackage{algorithmic}
\usepackage{natbib}
  
\title{Fed-REACT: Federated Representation Learning for Heterogeneous and Evolving Data
}

\author{
  Yiyue Chen \quad
  Usman Akram \quad Chianing Wang \thanks{ Toyota InfoTech Lab USA,  johnny.wang@toyota.com} \quad Haris Vikalo \\
  Department of Electrical and Computer Engineering \\
  University of Texas at Austin
}

\begin{document}
\maketitle

\begin{abstract}
Motivated by the high resource costs and privacy concerns associated with centralized machine learning, federated learning (FL) has emerged as an efficient alternative that enables clients to collaboratively train a global model while keeping their data local. However, in real-world deployments, client data distributions often evolve over time and differ significantly across clients, introducing heterogeneity that degrades the performance of standard FL algorithms. In this work, we introduce Fed-REACT, a federated learning framework designed for heterogeneous and evolving client data. Fed-REACT combines representation learning with evolutionary clustering in a two-stage process: (1) in the first stage, each client learns a local model to extracts feature representations from its data; (2) in the second stage, the server dynamically groups clients into clusters based on these representations and coordinates cluster-wise training of task-specific models for downstream objectives such as classification or regression. We provide a theoretical analysis of the representation learning stage, and empirically demonstrate that Fed-REACT achieves superior accuracy and robustness on real-world datasets.
\end{abstract}


\section{Introduction}

Distributed training of machine learning models has enabled significant advances across applications such as recommendation systems, image recognition, and conversational AI. Among distributed approaches, Federated Learning (FL) \citep{mcmahan2017communication} has garnered considerable attention for enabling collaborative, privacy-preserving training of a global model without requiring clients to share raw data.
However, classical algorithms like FedAvg \citep{mcmahan2017communication} assume independent and identically distributed (IID) data, which often does not reflect real-world scenarios. Since clients collect data asynchronously and in diverse environments, local datasets typically differ in both size and distribution, leading to statistical heterogeneity. This heterogeneity poses a major challenge for FL -- averaging updates from non-IID data can degrade global model performance and lead to poor local task accuracy \citep{zhao2018federated}. To address this, various techniques attempting to mitigate the effects of data heterogeneity have been proposed \citep{li2020federated}. Additionally, large-scale FL systems, such as those in cross-device settings, face further challenges including high communication costs and intermittent client availability. In response, client clustering and cluster-aware training strategies have been explored to improve both communication efficiency and learning performance \citep{mansour2020three, kim2021dynamic}.

In many real-world applications such as healthcare, autonomous driving, and finance, the data collected by clients evolves over time. While the above FL methods have proven effective for static heterogeneous data, most are not designed to handle evolving data characterized by an additional layer of heterogeneity arising from the temporal dimension. To tackle this, \citet{kim2021dynamic} proposed a framework that employs generative adversarial networks (GANs) to group users and dynamically update clusters without sharing raw data. However, this approach relies on clustering snapshots of temporal data, which can result in unstable cluster assignments and spurious detection of abrupt changes. An alternative is evolutionary clustering \cite{xu2014adaptive}, which incorporates historical information to produce smoother transitions and more stable cluster memberships over time.

The contribution of this work can be summarized as follows:

\begin{itemize}

\item \textbf{A novel FL framework for evolving heterogeneous data:} To our knowledge, this is the first work to formally study federated self-supervised learning under both heterogeneous and evolving data conditions.
In such settings, heterogeneity arises from two sources: inter-client distribution diversity, due to variations in data distributions across clients, and intra-client non-stationarity, stemming from temporal changes in data observed by each client. We propose Fed-REACT (\underline{Fed}erated learning method leveraging \underline{R}epresentation learning and \underline{E}volution\underline{A}ry \underline{C}lus\underline{T}ering), a novel two-phase framework. In the first phase, clients collaboratively learn meaningful feature representations via self-supervised learning. In the second phase, these representations are used to train task-specific models within dynamically evolving clusters of distributionally similar clients.

\item \textbf{Evolutionary clustering with adaptive forgetting:} To address intra-client heterogeneity, we introduce evolutionary clustering into federated learning and group clients based on the similarity of their task model weights. A key challenge lies in the high variability of weights, especially when local training is performed on small batches, which can destabilize clustering. To mitigate this, we propose an adaptive forgetting factor that enables clustering based on both current and historical model parameters. We further investigate strategies for aggregating cluster-specific models, including (a) time averaging and (b) weighted averaging with forgetting. These strategies are empirically evaluated in the results section.

\item \textbf{Theoretical analysis:} We provide theoretical analysis of the feature learning phase of Fed-REACT. Specifically, we define a global regret function for a linear feature model and analyze the performance of time-smoothed gradient descent on time-evolving data. We show that, with appropriate step size and smoothing window, the regret converges to a small value determined by the gradient projection error.

\end{itemize}

\subsection{Related Work}

Federated learning enables clients to collaboratively train a global model while preserving data privacy, as raw data remains local throughout the training process. However, statistical heterogeneity, arising from data collected across diverse times and locations, poses significant challenges. This often leads to degraded model performance and has motivated extensive research into strategies for mitigating the effects of data non-IIDness.

Self-supervised learning (SSL) has emerged as a promising approach for tackling data heterogeneity in distributed settings, particularly when labeled data is scarce or imbalanced \citep{wang2022does}. SSL typically involves a two-stage process: learning feature representations from unlabeled data, followed by training task-specific models using those features. While SSL has been widely adopted in static data domains such as vision, language, and video, its application to temporal or streaming data remains limited \citep{chen2020simple, chen2021exploring, chen2024fed}.


Several recent efforts have focused on learning representations for time series. \citet{fortuin2018som} and \citet{franceschi2019unsupervised} introduced unsupervised temporal embedding techniques, the latter using causal dilated convolutions and time-based negative sampling. \citet{wu2022timesnet} proposed TimesNet, which captures intra- and inter-periodic patterns in multivariate time series. Transformer-based approaches such as PatchTST \citep{nie2022time}, T-Loss \citep{fraikin2023t}, and TSLaNet \citep{eldele2024tslanet} aim to capture both short- and long-term dependencies via self-supervised pretraining. TimeLLM \citep{jin2023time} further reprograms time series data into textual representations for compatibility with large language models. Despite these advances, most of these methods assume centralized access to data, limiting their applicability in federated settings.

To address data heterogeneity in large-scale FL systems, several works have explored client clustering based on data similarity. \citet{ghosh2020efficient} introduced the Iterative Federated Clustering Algorithm (IFCA), which assigns cluster memberships via similarity coefficients. \citet{li2021federated} proposed Federated Soft Clustering (FLSC), showing that allowing clients to belong to multiple clusters can improve overall performance. More recently, \citet{zeng2024metaclusterfl} developed MetaClusterFL, a meta-learning approach for automatically determining the optimal number of clusters. While most of these methods assume static data distributions, \citet{mehta2023greedy} proposed FLACC, a greedy agglomerative clustering method based on client gradient updates.

In contrast, evolutionary clustering accounts for the temporal evolution of the objects being clustered, aiming to produce consistent cluster assignments over time. For example, \citet{xu2014adaptive} proposed the Adaptive Evolutionary Clustering Algorithm (AFFECT), which updates a weighted affinity matrix to ensure temporal smoothness. \citet{arzeno2019evolutionary} introduced Evolutionary Affinity Propagation (EAP), a factor-graph-based method that propagates cluster messages iteratively. While these approaches have been applied in domains such as social networks and time-evolving graphs, they have not, to our knowledge, been explored in federated learning settings. Our work is the first to incorporate evolutionary clustering into FL, enabling temporally stable client grouping in the presence of intra-client data drift.

\section{The Fed-REACT Framework}

\textbf{Problem setup and notation.} 
We consider a federated learning system with $n$ clients, where each client locally collects time series data with features $x \in \mathbb{R}^{d \times T}$ and label $y$, where $d$ is the feature dimension and $T$ is the maximum sequence length. A central server coordinates collaborative training by aggregating local updates and redistributing the global model to participating clients. The local dataset at client $i$ is denoted by $\mathcal{D}_i = \{(x, y)\}$, and its distribution may differ across clients, leading to inter-client data heterogeneity.
To address this, we adopt a self-supervised learning framework in which each client learns a shared feature extractor $f_\theta(\cdot)$, parameterized by $\theta$, to map input sequences from $\mathbb{R}^{d \times T}$ to a lower-dimensional representation space $\mathbb{R}^{\hat{d}}$. These representations are later used for downstream supervised tasks. Depending on the task type (e.g., regression or classification), a lightweight task-specific model $f_{\theta_{\text{task}}}(\cdot)$, parameterized by $\theta_{\text{task}}$, is trained on top of the learned representations using a small set of labeled samples.

\textbf{Phase I: Federated representation learning.} 
Our approach consists of two sequential phases: the first phase focuses on learning low-level feature representations via a globally shared encoder, while the second phase captures higher-level semantics and supports downstream task learning. This separation is motivated by the observation that low-level representations (such as edges in images or short-term patterns in time series) are often transferable across clients, even when their data distributions differ significantly. For instance, in image data, clients may hold samples from distinct classes or domains, but still share fundamental visual primitives such as edges or textures. Federated training of the encoder thus allows clients to collaboratively learn a generalizable representation space without requiring distributional alignment.

Specifically, the shared encoder $f(\cdot\, ; \theta)$, parameterized by $\theta$, is trained to minimize a contrastive loss function \citep{chen2020simple, franceschi2019unsupervised}. Let $x^{\text{ref}}$ be an anchor time series, $x^{\text{pos}}$ a positive sample from the same trajectory, and $\{x^{\text{neg}}_r\}_{r=1}^R$ a set of $R$ negative samples drawn from different trajectories. The loss is defined as:
\begin{align}
L_{\text{cl}} = -\log\big(\sigma\big(f(x^{\text{ref}}; \theta)^\top f(x^{\text{pos}}; \theta)\big)\big) 
- \sum_{r=1}^R \log\big(\sigma\big(-f(x^{\text{ref}}; \theta)^\top f(x^{\text{neg}}_r; \theta)\big)\big), \label{eqn1}
\end{align}
where $\sigma(\cdot)$ denotes the sigmoid function. This objective encourages alignment between the anchor and its positive sample, while pushing apart representations of the anchor and negatives. For time series data, positive samples are typically sub-sequences from the same trajectory, whereas negatives are drawn from unrelated sequences.

The full procedure is formalized as Algorithm~\ref{alg1}.
The proposed framework is agnostic to the choice of encoder architecture and can accommodate a variety of models capable of capturing temporal dependencies. In our experiments, we use a Causal CNN with exponentially dilated convolutions, following \citet{franceschi2019unsupervised}, due to its ability to model long-range temporal structure. 


\textbf{Phase II: Clustered task model learning.}
In the second phase, the focus shifts to downstream task learning. Since task models are intended to capture higher-level, task-specific features that are often tied to local data characteristics, it is beneficial for clients with similar data distributions to collaboratively train shared task model weights. The architecture of the task model depends on the downstream task: we use support vector machines (SVMs) for classification and a linear layer with $\ell_2$ loss for regression. Due to privacy constraints, clients cannot share label distributions. Instead, we perform client clustering based on their task model weights, which implicitly encode information about the underlying data. A simple yet effective approach is for the server to periodically collect task model weights from clients and apply agglomerative hierarchical clustering to group them. The detailed procedure is formalized in the supplementary material. 

\vspace{-0.025in}

\textbf{Limitations of snapshot clustering.}
The snapshot-based clustering approach described above only considers task model weights from a single training round and thus fails to account for temporal correlations in evolving client data. Additional challenges arise due to: (a) the relatively small number of labeled samples available for training task models compared to the unlabeled samples used for encoder training, and (b) the limited retention window for labeled data on many clients, where older samples are often deleted or overwritten by newly collected data. As a result, task models trained in a single round may not accurately reflect clients’ underlying data distributions, leading to unstable or incorrect clustering assignments.

\vspace{-0.025in}

\begin{algorithm} \small
\caption{Fed-REACT Phase 1: Encoder training} \label{alg1}
\begin{algorithmic}[1]
\STATE \textbf{Input:} Number of rounds $T$, number of clients $K$, initialized global encoder parameters $\mathbf{\theta}_0$
\FOR{each round $t = 1, 2, ..., T$}
    \FOR{each client $k = 1,2,...,K$ }
        \STATE Client $k$ downloads current global model parameters $\mathbf{\theta}_{t-1}$  
        \STATE Client $k$ updates parameters $\mathbf{\theta}_t^k$ using local time series data 
        \STATE Client $k$ uploads updated parameters $\mathbf{\theta}_t^k$ to the server
    \ENDFOR
    \STATE Server aggregates collected updates as
    \[
    \mathbf{\theta}_t = \sum_{k=1}^K \frac{n_k}{n} \mathbf{\theta}_t^k,
    \]
    where $n_k$ is the number of samples on client $k$ and $n = \sum_{k=1}^K n_k$
\ENDFOR
\end{algorithmic} \normalsize
\end{algorithm}

\textbf{Temporal clustering and task model aggregation.}
To make the clustering phase of Fed-REACT robust to temporal variation and training noise, we adopt the Adaptive Evolutionary Clustering framework of \citet{xu2014adaptive}, which allows cluster memberships to evolve over time. Let $\psi_t$ denote the (unobserved) ground-truth similarity matrix among clients at time $t$, capturing underlying client relationships. The observed similarity matrix $W_t$ is a noisy approximation of $\psi_t$,
\begin{align}
W_t = \psi_t + N_t, \label{eqn2}
\end{align}
where $[W_t]_{i,j}$ represents the cosine similarity between the vectorized task model weights of clients $i$ and $j$, and $N_t$ denotes noise. The evolutionary clustering method of \citet{chakrabarti2006evolutionary} smooths these similarities over time using a fixed forgetting factor $a$ as
$\hat{\psi}_t = a \hat{\psi}_{t-1} + (1-a) W_t$,
with initial condition $\hat{\psi}_0 = 0$. The AFFECT algorithm \citep{xu2014adaptive} improves upon this by adaptively estimating the forgetting factor $a_t$ at each time step, yielding
\begin{align}
\hat{\psi}_t = a_t \hat{\psi}_{t-1} + (1 - a_t) W_t. \label{eqn-affect}
\end{align}
Once the smoothed similarity matrix $\hat{\psi}_t$ is computed, cluster assignments are determined using agglomerative hierarchical clustering as previously described.

\vspace{-0.025in}

Tracking the temporal evolution of client clusters enables the aggregation of cluster-specific task model weights across rounds. We investigate two strategies for combining task model parameters:

\vspace{-0.025in}

\begin{enumerate}
\item \textbf{Approach 1: Simple Temporal Averaging (A1).} Task model parameters are averaged across rounds via
\begin{align}
\hat{\theta}^{c}_{\text{task}, t+1} = \frac{t}{t+1} \hat{\theta}^{c}_{\text{task}, t} + \frac{1}{t+1} \theta^{c}_{\text{task}, t}, \label{eqn_algor_2_1}
\end{align}
where $\theta^{c}_{\text{task}, t}$ is the parameter estimate based solely on round $t$, and $\hat{\theta}^{c}_{\text{task}, t}$ is the cumulative estimate from prior rounds. Initialization is given by $\hat{\theta}^{c}_{\text{task}, 1} = \theta^{c}_{\text{task}, 1}$.

\item \textbf{Approach 2: Weighted Averaging with Forgetting (A2).} This method leverages the adaptive forgetting factor $a_t$ to recursively update the model according to
\begin{align}
\hat{\theta}^{c}_{\text{task}, t+1} = a_t \hat{\theta}^{c}_{\text{task}, t} + (1 - a_t) \theta^{c}_{\text{task}, t}. \label{eqn_algor_2_2}
\end{align}
\end{enumerate}

The full procedure for this phase is summarized in Algorithm~\ref{alg3}.

\begin{algorithm}[h]\small
\caption{Fed-REACT Phase 2: Task model training with evolutionary clustering} \label{alg3}
\begin{algorithmic}[1]
\STATE \textbf{Input:} Number of rounds $T_{task}$, number of clients $K$, 
cluster number $C$\footnote{Add a footnote on how to determine the cluster number in practice}, trained encoder $\mathbf{\theta}_T$
\FOR{ each round $t=1,2,..., T_{task}$}

\FOR {client $k=1,2,..,K$}
\STATE Client $k$ trains the task model on randomly sampled local dataset $\mathcal{M}^k_{t}$
\STATE Client $k$ uploads the parameters $\mathbf{\theta}_{task, t}^k$ 
 of the task model to the server
\ENDFOR

\STATE Server clusters clients based on the weights of the task models $\{ \mathbf{\theta}_{task, t}^k\}_{k=1}^K$ using the AFFECT algorithm to obtain the cluster memberships of $C$ clusters, $\{\mathcal{S}_t^c \}_{c=1}^C $, and adaptive forgetting factor $a_t$. 

\FOR{cluster $c=1,2,..,C$ }
    \STATE Server aggregates the task models of all clients within cluster $\mathcal{S}_t^c$ according to
    \[
    \mathbf{\theta_{task, t}^c} = \sum_{k \in  \mathcal{S}_t^c} \frac{|\mathcal{M}^k_t|}{\sum_{j \in \mathcal{S}_t^c} |\mathcal{M}^j_t|} \mathbf{\theta}^k_{task, t}
    \]     

    \IF{$t \geq T_{task}$ or $\mathcal{S}_t^c=\mathcal{S}_{t-1}^c$}
    

    \STATE Compute $\mathbf{\hat{\theta}_{task, t}}^{c}$ using Approach A1 or A2
    
    \STATE Server transmits $\mathbf{\hat{\theta}_{task, t}}^{c}$ to all clients $k \in \mathcal{S}_t^c $

    \ENDIF
\ENDFOR

\ENDFOR
\end{algorithmic} \normalsize
\end{algorithm}

\section{Theoretical Analysis}

In this section, we provide theoretical insight into the first phase of Fed-REACT algorithm, i.e., representation learning on heterogeneous temporal data. In particular, we analyze the convergence of a time-varying objective function under the assumption that each client trains a linear encoder via a dynamic time-smoothed gradient method. For the sake of analytical tractability, we consider a self-supervised learning (SSL) formulation obtained by simplifying \eqref{eqn1} that utilizes local loss function at client $k$ 
\[
f_{k}(\theta) = -\mathbb{E}[(\theta(x_{k, i}) +\xi_{k, i})^T(\theta(x_{k, i}) +\xi'_{k, i})] + \frac{1}{2} \|\theta^T \theta \|^2,
\]
with $\xi_{k, i}$ and $\xi'_{k, i}$ denoting random noise added to the data sample $x_{k, i} $. The corresponding global objective is defined as 
\[
f(\theta) = \sum_{k=1}^K \frac{|\mathcal{D}_k|}{|\mathcal{D}|}f_{ k}(\theta),
\]
where $\mathcal{D}_k$ is the dataset at client $k$ and $|\mathcal{D}| = \sum_k |\mathcal{D}_k|$ denotes the total number of data points. This objective is a variant of the 
contrastive loss \eqref{eqn1}, where the normalization over negative samples is replaced by a regularization term. Minimizing 
$f(\theta)$ is 
equivalent to finding $\arg \min_\theta \|\bar{X} - \theta^T \theta \|^2$, 
where $\bar{X} = \sum_k \frac{|\mathcal{D}k|}{|\mathcal{D}|} X_k$ is the global empirical covariance matrix, and $X_k = \frac{1}{|\mathcal{D}_k|}\sum_{i=1}^{|\mathcal{D}_k|} x_{k, i}x_{k, i}^T $ is the empirical covariance matrix of client $k$'s data \citep{wang2022does}.

To proceed, we make the following assumptions regarding the time-varying local function. 

{\bf Assumption 3.1.} (a) Loss function $f_{t, i}$ is bounded above by $M$ for all clients $i$ and times $t$. (b) Loss function $f_{t, i}$ is $L$-Lipschitz and $\beta$-smooth. (c) The stochastic gradient $\Tilde{\nabla} f(\cdot)$ is unbiased and its standard deviation 
is bounded above by $\sigma$. The error between the projected stochastic gradient $Proj \Tilde{\nabla} f(\cdot)$ and the stochastic gradient $ \Tilde{\nabla} f(\cdot)$ is $\epsilon_{proj} = Proj (\Tilde{\nabla} f(\cdot)) -\Tilde{\nabla} f(\cdot) $ with $\| \epsilon_{proj} \|^2 \leq \epsilon^2$.

\citet{jin2017escape} have shown that the form of the objective function studied in our work is $16\Gamma$-smooth within the region $\{x| \| x\|^2 \leq \Gamma\}$ for $\Gamma \geq \lambda_1(\bar{X})$, where $\lambda_1(\bar{X})$ is the largest eigenvalue of the global covariance matrix. This implies that the Lipschitz and smoothness assumptions (a) and (b) are readily satisfied in our setting. Moreover, the projected gradient step used in the Fed-REACT algorithm ensures that the iterate $x$ remains within this region at all times. Assumption (c) is standard in stochastic optimization.

Next, we describe the update rule applied by client $k$ during the encoder learning phase. Specifically, the updates follow a time-smoothed gradient descent scheme \citep{aydore2019dynamic}, where the local update is given by
\[
\theta_{t+1, k} = \theta_t - \frac{\eta}{W}\sum_{j=0}^{w-1} \gamma^j Proj \Tilde{\nabla} f_{t-j, k}(\theta_{t-j}),
\]
and the corresponding global update is computed as
\[
    \theta_{t+1} = \frac{1}{n}\sum_{k=1}^K \theta_{t+1, k}.
\]
Here, $\eta$ is the step size, $w$ is the size of the temporal smoothing window, $\gamma \in (0,1]$ is the decay factor, and $W = \sum_{j=0}^{w-1} \gamma^j$ is the normalization constant. The projection operator ensures that updates remain within the bounded region specified in Assumption 3.1.

We define the local regret at client $k$ and the global regret as
\[
S_{t, w, \gamma, k}(\theta_t) = \frac{1}{W}\sum_{j=0}^{w-1} \gamma^j f_{t-j, k}(\theta_{t-j}), \;\; 
S_{t, w, \gamma}(\theta_t) = \frac{1}{K}\sum_{k=1}^K \frac{1}{W}\sum_{j=0}^{w-1} \gamma^j f_{t-j, k}(\theta_{t-j}),
\]
respectively. It follows from the assumptions that the smoothed gradient estimates are unbiased, i.e., $\mathbb{E} [\Tilde{\nabla} S_{t, w, \gamma}(\theta_t) | \theta_t ] = \nabla S_{t, w, \gamma}(\theta_t)$, 
$\mathbb{E} [\Tilde{\nabla} S_{t, w, \gamma, k}(\theta_t) | \theta_t ] =\nabla S_{t, w, \gamma, k}(\theta_t)$. Moreover, the variance of the local smoothed gradient estimator is bounded as
\begin{align*}
    \mathbb{E} [\Tilde{\nabla} S_{t, w, \gamma, k}(\theta_t) - \nabla S_{t, w, \gamma, k}(\theta_t)  | \theta_t ] & \leq \frac{\sigma^2(1-\gamma^{2w})}{W^2(1-\gamma^2)}.
\end{align*}
With this notation in place, we can now state the main convergence result (Theorem~\ref{theorem1}); the proof is provided in the supplementary material.




\begin{theorem}\label{theorem1}
Suppose Assumption 3.1 holds. Let the step size be set to $\eta = 1/\beta$, and consider the limit as the smoothing decay parameter $\gamma \to 1^{-}$. Then, the average squared gradient norm of the global smoothed objective satisfies
\begin{align*}
    \lim_{\gamma \to 1^{-}} \frac{1}{T} \sum_{t=1}^T  \| \nabla S_{t, w, \gamma}(\theta_t) \|^2 \leq \frac{64\beta M + 2\sigma^2}{W}+ \frac{5}{8}\epsilon^2.
\end{align*}
\end{theorem}
This result implies that, for sufficiently large smoothing window size $w$ (i.e., large $W$) and appropriate choice of step size $\eta$, the dominant term in the upper bound becomes the projection error between the true stochastic gradient and its projected counterpart. Consequently, the global regret converges to a small value determined primarily by the gradient projection error $\epsilon^2$.


\section{Experiments}
In Section~\ref{section4-1}, we compare Fed-REACT with supervised learning baselines on a range of time series tasks, demonstrating its superior performance. Section~\ref{section4-2} evaluates Fed-REACT against clustered FL methods, highlighting the benefits of evolutionary clustering under non-stationary local data distributions. Section~\ref{section4-3} presents an ablation study analyzing the impact of key design choices.

\subsection{Performance Comparison with Supervised FL Methods}\label{section4-1}

We compare the performance of Fed-REACT to that of several time series models embedded in supervised federated learning (FL) frameworks. The baseline models include LSTM \cite{hochreiter1997long}, TimesNet \cite{wu2022timesnet}, PatchTST \cite{nie2022time}, and Causal CNN \cite{franceschi2019unsupervised}. These are combined with state-of-the-art FL algorithms designed to handle data heterogeneity: FedAvg \citep{mcmahan2017communication}, FedProx \cite{li2020federated}, Ditto \cite{li2021ditto}, and Adaptive Personalized Federated Learning (APFL) \cite{deng2020adaptive}. For Fed-REACT, the encoder is a Causal CNN and the task model is either an SVM (for classification) or a linear regressor (for regression). For implementation details, please see the supplementary material.
\begin{table*}[htbp]
  \centering
  \resizebox{\linewidth}{!}{%
  \begin{tabular}{c|c|cccc|cccc|cccc|cccc}\toprule
      & & \multicolumn{4}{c}{LSTM} &  \multicolumn{4}{c}{TimesNet} &  \multicolumn{4}{c}{PatchTST} & \multicolumn{4}{c}{Causal CNN} \\ 
    Dataset & Fed-REACT & FedAvg &  FedProx &  Ditto & APFL & FedAvg &  FedProx &  Ditto & APFL & FedAvg &  FedProx &  Ditto & APFL & FedAvg &  FedProx &  Ditto & APFL\\ \midrule
    RTD 10 clients (acc) & \textbf{0.992} & 0.732 & 0.804 & 0.863 & 0.828 & 0.793 & 0.883 & 0.863 & 0.755 & 0.918 & 0.903 & 0.991 & 0.991 & 0.982 & 0.988 & 0.989 & 0.990\\ \midrule
    RTD 50 clients (acc) & \textbf{0.988} & 0.868 & 0.835 & 0.914 & 0.867 & 0.827 &0.831 & 0.801 & 0.795 & 0.756 & 0.769 & 0.987 & 0.923 & 0.986 & 0.984 & 0.8953 & 0.650 \\ \midrule
    EEG 10 clients (acc) & \textbf{0.796} & 0.505 & 0.511 & 0.507 & 0.499 & 0.572 & 0.574 & 0.567 & 0.578 & 0.533 & 0.535 & 0.576 & 0.532 & 0.559 & 0.605 & 0.516 & 0.606 \\ \midrule
    SUMO EV (RMSE) & \textbf{1.3} & 43.2 & 42.3 &  42.0 & 42.7 & 35.4 & 34.3 & 37.1 & 35.0 & 21.9 & 21.7 & 28.3 & 20.1 & 39.8 & 38.2 & 40.1& 38.6 \\
    \bottomrule
  \end{tabular}
}
  \caption{\small Comparison of Fed-REACT against supervised FL methods. For RTD and EEG, average test accuracy is reported; for SUMO EV, the metric is RMSE. For Fed-REACT ($T_{\text{task}}=1$), cluster-specific model accuracy is computed as $\frac{1}{K} \sum_{C_i} \sum_{k \in C_i} \mathrm{Acc}_{C_i}(\mathcal{D}_{k,\text{test}})$, where $K$ is the number of clients and $\mathrm{Acc}_{C_i}(\mathcal{D}_{k,\text{test}})$ is the accuracy of cluster $C_i$’s model on client $k$’s test data. RMSE is computed analogously for regression.} 
  \label{tab:ssl_vs_sl}
  \vspace{-3mm}
\end{table*}

The first dataset is RTD \cite{alam2020trajectory}, which consists of 3D air-writing trajectories for $2000$ samples of digits 0-9. The trajectories vary in length, with the longest containing 100 timesteps; shorter sequences are zero-padded to match this length. The data is partitioned into three clusters based on sequence composition, using a Dirichlet distribution with parameter $\beta = 0.1$ to induce high heterogeneity. In the setting with $K = 10$ clients, Clusters 1 and 2 contain 3 clients each, and Cluster 3 contains 4; in the $K = 50$ setting, the cluster sizes are 16, 16, and 18, respectively. Each client receives 2400 samples uniformly drawn from its assigned cluster, with a 90/10 train-test split.
Performance on the RTD dataset is shown in the first two rows of Table~\ref{tab:ssl_vs_sl}. In the 10-client setting, Fed-REACT significantly outperforms all baselines. In the 50-client scenario, it maintains its advantage, achieving the highest accuracy among all methods. Among the baselines, PatchTST performs best, particularly when combined with the Ditto FL framework.


The second dataset is EEG \citep{10.1093/pnasnexus/pgae076}, partitioned into three clusters with equal numbers of samples. Each sample corresponds to a single EEG trial: a 5-second recording across $26$ channels, pre-processed to a maximum sequence length of $70$. Cluster 1 contains only left-hand motor imagery samples, Cluster 2 contains only right-hand samples, and Cluster 3 contains both. The system consists of 10 clients, with Clusters 1, 2, and 3 assigned 3, 3, and 4 clients, respectively. Each client receives $1000$ data samples, with $810$ used for local training. Performance on the EEG dataset is reported in the third row of Table~\ref{tab:ssl_vs_sl}. Fed-REACT achieves an accuracy of $0.796$, outperforming all baselines.

The final test is on the Simulation of Urban Mobility (SUMO) dataset \citep{krajzewicz2012recent}, which captures vehicle behavior under varying environmental and geographic conditions, including temperature, humidity, elevation, and location. Unlike the previous experiments, the task here is regression: predicting the percentage of battery life remaining from a 100-step multivariate time series input. This dataset is heterogeneous in both sample size and data distribution. Some vehicles have as few as $100$ training samples, while others have over $1000$. Even vehicles of the same type exhibit different battery usage patterns, making the client clustering problem particularly challenging. Each time series includes features such as latitude, longitude, elevation, temperature, speed, maximum speed, acceleration, and vehicle type. All features are normalized before training. The data is split into $90$\% training and $10$\% testing; the test set includes $50$ vehicles. The final row of Table~\ref{tab:ssl_vs_sl} reports the root mean square error (RMSE) averaged across clients. Fed-REACT achieves the lowest RMSE, confirming that its learned representations extract more meaningful features than those obtained by supervised FL baselines. Since the number of clusters $C$ is not known in advance, we search over a range of values to select the one yielding the best performance. Additional details are in the supplementary material.

\textbf{Computational and communication complexity.} In addition to predictive performance, it is also important to consider the computational and communication efficiency of Fed-REACT relative to the baselines. The computational efficiency of Fed-REACT is closely tied to the choice of encoder architecture. In our experiments, we use a Causal CNN, which offers a large receptive field while scaling linearly with sequence length. In contrast, PatchTST scales quadratically with input length, while TimesNet incurs additional overhead from converting time series into frequency-domain image representations. These architectural differences, combined with Fed-REACT's two-stage design where shared representations are learned once and reused for lightweight, cluster-specific task models, lead to lower computational and communication costs compared to fully supervised FL baselines (for more details, please see Section 9 of the supplementary material). Despite this efficiency, Fed-REACT consistently achieves higher accuracy across all tasks and datasets.

\subsection{Evaluation of Evolutionary Clustering}\label{section4-2}
In this section, we evaluate Fed-REACT in clustered FL settings, highlighting the role of evolutionary clustering in adapting to non-stationary client distributions. The baseline clustered FL methods considered are IFCA \cite{ghosh2020efficient}, Federated Learning with Soft Clustering (FLSC) \cite{li2021federated}, and FLACC \cite{mehta2023greedy}.
To isolate the effect of clustering from that of model aggregation, we also consider variants in which task model parameters are aggregated in a memoriless fashion i.e., without using past values of the task model parameters. In these settings, clients are grouped using either snapshot or evolutionary clustering strategies. We refer to the resulting methods as Snapshot Clustering with Memoriless Model Aggregation (SC+MMA) and Evolutionary Clustering with Memoriless Model Aggregation (EC+MMA). Performance is evaluated using accuracy and Rand score. The Rand score quantifies clustering quality by comparing predicted clusters with the ground-truth partition. It is defined as $\frac{TP + TN}{TOT}$, where $TP$ is the number of client pairs correctly assigned to the same cluster, $TN$ is the number correctly assigned to different clusters, and $TOT$ is the total number of client pairs.

\begin{table*}[t]
  \centering
  \normalsize
  \resizebox{0.95\linewidth}{!}{
  \begin{tabular}{c|cc|c|c|c|c|c}\toprule
    Dataset & Fed - REACT w/ A1 & Fed - REACT w/ A2 &  SC + MMA & EC + MMA & IFCA & FLSC &  FLACC \\ \midrule
    RTD - 10 clients (Strategy 1) & \textbf{0.918} & 0.870  & 0.827 & 0.826 & 0.883  & 0.887 &  0.876  \\ \midrule
    RTD - 100 clients (Strategy 1) & 0.790  & \textbf{0.791}  & 0.724 & 0.733  & 0.701 & 0.695 & 0.693  \\ \midrule
    RTD - 100 clients (Strategy 2) & 0.856 & \textbf{0.858} & 0.803 & 0.838 & 0.684 & 0.581 & 0.408 \\ \midrule
    EEG - 10 clients & 0.802 &\textbf{0.808}  & 0.799 & 0.800 & 0.513 & 0.513 & 0.565 \\ \midrule
  \end{tabular}
  }
  \caption{\small Accuracy of Fed-REACT compared to clustered FL baselines across RTD and EEG datasets. The accuracy is computed by averaging cluster-specific model accuracies.}
  \label{tab:cluster_acc}
\end{table*}

\textbf{RTD dataset.} For this dataset, we construct non-stationary label distributions using two strategies:

\textit{Strategy 1.} To simulate non-stationarity, each cluster $c$ alternates between two distinct label distributions, $p_{c,\text{major}}$ and $p_{c,\text{minor}}$, drawn via Dirichlet sampling with overlapping support. The switching behavior is governed by a Markov process. Let $z_{c,t}$ denote the latent state of cluster $c$ at round $t$, with transition probabilities $\Pr(z_{c,t} = 1 \mid z_{c,t-1} = 0) = \lambda_1$ and $\Pr(z_{c,t} = 1 \mid z_{c,t-1} = 1) = \lambda_2$. The label distribution for cluster $c$ at time $t$ is then given by
\[p_{c,t}=(1-z_{c, t})p_{c,\text{major}}+(z_{c, t})p_{c,\text{minor}}.\]
We evaluate two settings: 10 clients grouped into 3 clusters over 100 rounds, and 100 clients across 3 clusters over 200 rounds. In both, each client receives $|\mathcal{M}^k_t| = 64$ training samples per round. Results are reported in Table \ref{tab:cluster_acc} for $\lambda_1 = 0.85$ and $\lambda_2 = 0.15$, with additional ablations over $\lambda_1$ and $\lambda_2$ provided in the supplementary material.

\textit{Strategy 2.} In contrast to Strategy 1, where clusters switch between two fixed distributions, this approach generates non-stationarity by continuously sampling label distributions from fixed, non-overlapping supports. Specifically, each cluster is assigned a distinct label subset: 3 classes for Cluster 1, 3 for Cluster 2, and 4 for Cluster 3. At each round, a probability vector is sampled uniformly from the simplex over the cluster’s label support. Additionally, each client has a small probability (0.05) of temporarily adopting the label distribution of one of the other two clusters, introducing stochastic cross-cluster drift. Experimental results for 100 clients over 200 rounds are presented in Table~\ref{tab:cluster_acc}.


Figure~\ref{fig:nonstat100_85} demonstrates that Fed-REACT correctly groups clients even as their local data distributions evolve over communication rounds, whereas baseline clustering methods struggle to recover the ground-truth structure. The first three rows of Table~\ref{tab:cluster_acc} further show that Fed-REACT consistently achieves higher accuracy than competing clustered FL schemes.

\begin{figure*}[t]
  \centering
  \begin{subfigure}[t]{0.48\textwidth}
    \centering
    \includegraphics[width=\linewidth]{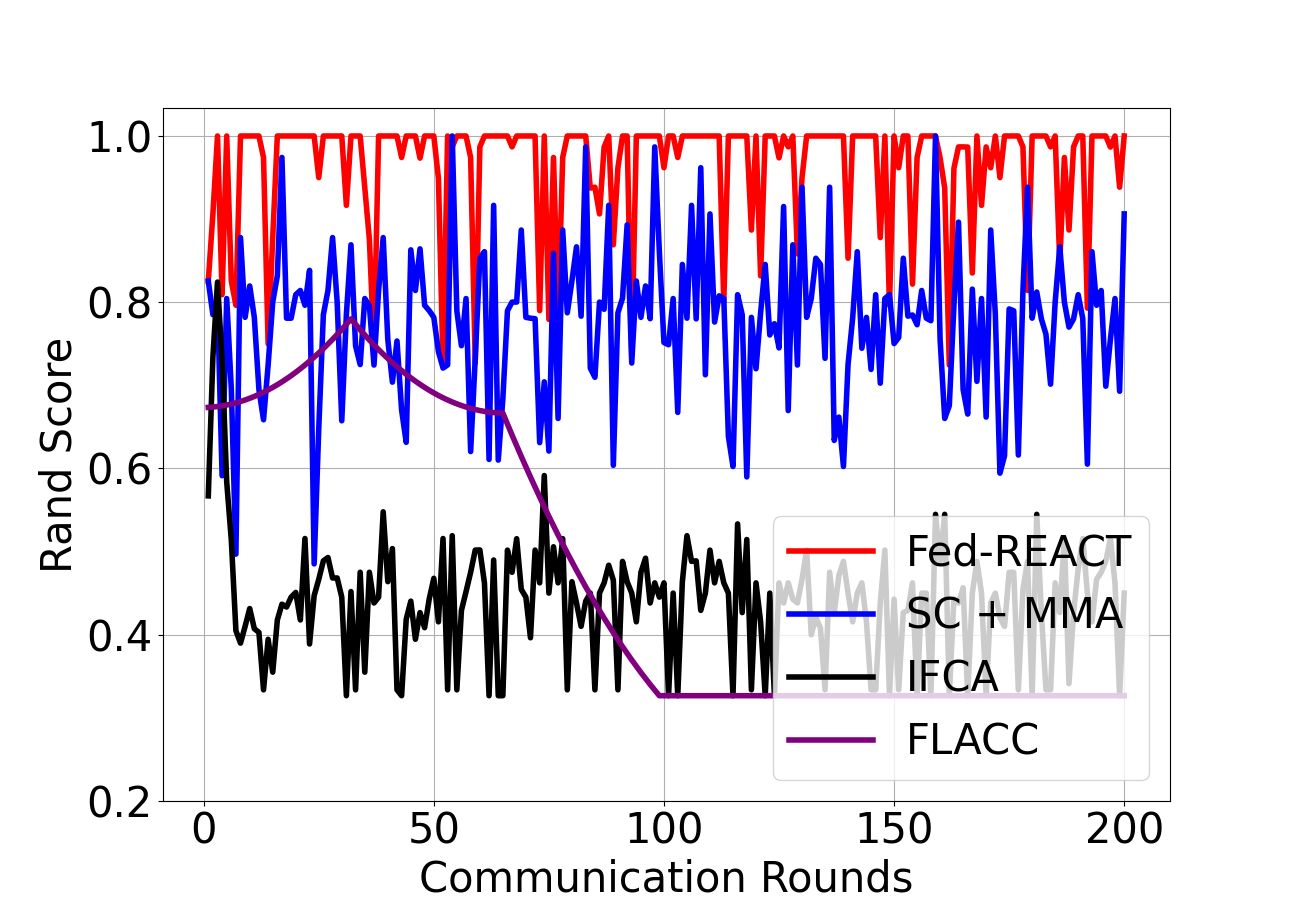}
    \caption{RTD (Strategy 1), 100 clients, $|\mathcal{M}^k_t| = 64$, $T_{\text{task}} = 200$.}
    \label{fig:nonstat100_85}
  \end{subfigure}
  \hfill
  \begin{subfigure}[t]{0.48\textwidth}
    \centering
    \includegraphics[width=\linewidth]{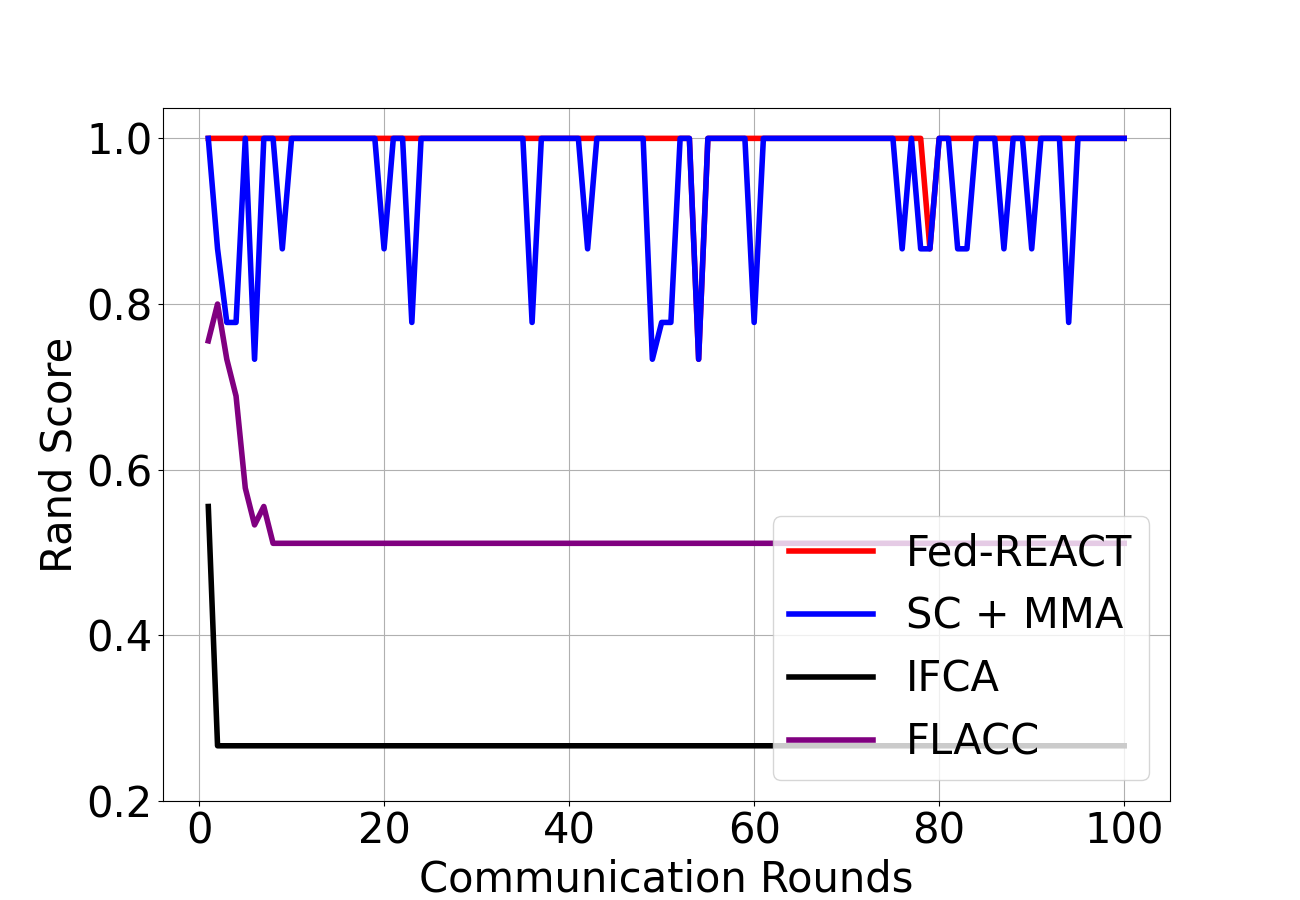}
    \caption{EEG, 10 clients, $|\mathcal{M}^k_t| = 384$, $T_{\text{task}} = 100$.}
    \label{fig:eegrand}
  \end{subfigure}
  \caption{\small Rand scores for client clustering under distribution drift (RTD, EEG).}
  \label{fig:rand_scores_combined}
\end{figure*}

\textbf{EEG dataset.}
The EEG dataset is inherently non-stationary, as neural signals from the brain can vary over time even for the same motor imagery task. We conduct the experiments as follows: at each round $t$, client $k$ trains its task model using $|\mathcal{M}^k_t| = 384$ labeled samples, for a total of 100 communication rounds. Rand scores for the different clustering methods are shown in Figure~\ref{fig:eegrand}. Fed-REACT rapidly recovers the correct cluster memberships, achieving a perfect Rand score of 1. In contrast, competing methods either fail to discover the correct clustering structure (IFCA, FLACC) or exhibit unstable cluster assignments as the data evolves. Task model accuracies are reported in the final row of Table~\ref{tab:cluster_acc}, where Fed-REACT again outperforms all baseline clustered FL methods.

\textbf{Computational and communication complexity.}
Fed-REACT is more communication-efficient than baseline clustered FL methods such as IFCA and FLSC, which require client-side cluster assignments and thus must transmit all cluster-specific models to each participating client. In contrast, Fed-REACT centralizes clustering at the server, resulting in constant communication cost with respect to the number of clusters. This efficiency comes with a trade-off: while IFCA and FLSC perform only simple model averaging at the server, Fed-REACT executes evolutionary clustering, whose complexity scales polynomially with the number of clients. For IFCA and FLSC, both clustering and communication costs scale linearly with the number of clusters and model size. Fed-REACT shifts the clustering burden to the server but avoids client-side computation and reduces communication overhead, especially in systems with many clusters.
\subsection{Ablation Analysis: Cluster Count and Data Heterogeneity}\label{section4-3}

\begin{table*}[t]
  \centering
  \small
  \resizebox{0.95\linewidth}{!}{
  \begin{tabular}{c|c|c|c|c|c|c|c}\toprule
    Number of Clusters & Fed - REACT w/ A1 & Fed - REACT w/ A2 &  SC + MMA & EC + MMA & IFCA & FLSC &  FLACC \\ \midrule
    4  & \textbf{0.8146} & 0.8145 & 0.7906 & 0.8058 & 0.7663  &  0.7829 & 0.6585 \\ \midrule
    5  & \textbf{0.7920}  & 0.7914  & 0.7608  & 0.7826  & 0.7550  & 0.7504  &  0.6176  \\ \midrule
    6  & \textbf{0.8445}  & 0.8425  & 0.8120  & 0.8368  & 0.8316 & 0.7755 & 0.6018  \\ \midrule
    7  & \textbf{0.7682}  &  0.7674 & 0.7388  & 0.7550  & 0.7450 & 0.7599 &  0.6604 \\ \bottomrule
    \toprule 
    Dirichlet $\beta $  & Fed - REACT w/ A1 & Fed - REACT w/ A2 &  SC + MMA & EC + MMA & IFCA & FLSC &  FLACC \\ \midrule
    $\beta = 0.25$ & \textbf{0.872} & 0.871 & 0.868 & 0.868 & 0.872 & 0.761 & 0.620 \\ \midrule
    $\beta = 0.5$ & \textbf{0.816} & 0.815 & 0.809 & 0.809 & 0.711 & 0.735 & 0.629 \\ \midrule
    $\beta = 2$ & \textbf{0.742} & 0.738 & 0.712 & 0.721 & 0.730 & 0.721 & 0.635 \\
    \bottomrule
  \end{tabular}
  }
  \caption{\small Ablation study on Fed-REACT accuracy under varying numbers of ground-truth clusters and levels of data heterogeneity ($\beta$).}
  \label{tab:numcluster_acc}
\end{table*}

We conduct ablation experiments on the RTD dataset to evaluate the sensitivity of Fed-REACT to two key factors: the number of clusters and the Dirichlet parameter $\beta$, which controls the degree of data heterogeneity. All experiments use 100 clients. To assess the effect of cluster granularity, we fix $\beta = 0.5$ and vary the number of ground-truth clusters $C \in \{4, 5, 6, 7\}$. To evaluate robustness to heterogeneity, we fix the number of clusters and vary $\beta \in \{0.25, 0.5, 2\}$, where smaller $\beta$ indicates greater distributional skew. As shown in Table~\ref{tab:numcluster_acc}, Fed-REACT consistently outperforms baseline clustering methods across all tested configurations, demonstrating strong robustness to both cluster granularity and data heterogeneity.

\section{Conclusion}

In this paper, we addressed the problem of federated self-supervised representation learning combined with (semi-)personalized task model training. To our knowledge, this is the first work to study this problem in the context of heterogeneous, evolving time series data. We proposed Fed-REACT, a two-phase framework that aggregates representation models globally and performs cluster-wise aggregation of task models—such as SVMs for classification and dense layers for regression. We provided theoretical analysis of the representation learning phase and demonstrated, through experiments on RTD, EEG, and SUMO EV datasets, that Fed-REACT consistently outperforms supervised FL baselines. Future work may explore fully decentralized settings in which clients must learn from evolving data without assistance from a central server.

\bibliography{neurips_paper}

\begin{thebibliography}{31}
\providecommand{\natexlab}[1]{#1}
\providecommand{\url}[1]{\texttt{#1}}
\expandafter\ifx\csname urlstyle\endcsname\relax
  \providecommand{\doi}[1]{doi: #1}\else
  \providecommand{\doi}{doi: \begingroup \urlstyle{rm}\Url}\fi

\bibitem[Alam et~al.(2020)Alam, Kwon, Alam, Abbass, Imtiaz, and Kim]{alam2020trajectory}
Md~Shahinur Alam, Ki-Chul Kwon, Md~Ashraful Alam, Mohammed~Y Abbass, Shariar~Md Imtiaz, and Nam Kim.
\newblock Trajectory-based air-writing recognition using deep neural network and depth sensor.
\newblock \emph{Sensors}, 20\penalty0 (2):\penalty0 376, 2020.

\bibitem[Arzeno and Vikalo(2019)]{arzeno2019evolutionary}
Natalia~M Arzeno and Haris Vikalo.
\newblock Evolutionary clustering via message passing.
\newblock \emph{IEEE Transactions on Knowledge and Data Engineering}, 33\penalty0 (6):\penalty0 2452--2466, 2019.

\bibitem[Aydore et~al.(2019)Aydore, Zhu, and Foster]{aydore2019dynamic}
Sergul Aydore, Tianhao Zhu, and Dean~P Foster.
\newblock Dynamic local regret for non-convex online forecasting.
\newblock \emph{Advances in neural information processing systems}, 32, 2019.

\bibitem[Chakrabarti et~al.(2006)Chakrabarti, Kumar, and Tomkins]{chakrabarti2006evolutionary}
Deepayan Chakrabarti, Ravi Kumar, and Andrew Tomkins.
\newblock Evolutionary clustering.
\newblock In \emph{Proceedings of the 12th ACM SIGKDD international conference on Knowledge discovery and data mining}, pages 554--560, 2006.

\bibitem[Chen et~al.(2020)Chen, Kornblith, Norouzi, and Hinton]{chen2020simple}
Ting Chen, Simon Kornblith, Mohammad Norouzi, and Geoffrey Hinton.
\newblock A simple framework for contrastive learning of visual representations.
\newblock In \emph{International conference on machine learning}, pages 1597--1607. PMLR, 2020.

\bibitem[Chen and He(2021)]{chen2021exploring}
Xinlei Chen and Kaiming He.
\newblock Exploring simple siamese representation learning.
\newblock In \emph{Proceedings of the IEEE/CVF conference on computer vision and pattern recognition}, pages 15750--15758, 2021.

\bibitem[Chen et~al.(2024)Chen, Vikalo, and Wang]{chen2024fed}
Yiyue Chen, Haris Vikalo, and Chianing Wang.
\newblock Fed-qssl: A framework for personalized federated learning under bitwidth and data heterogeneity.
\newblock In \emph{Proceedings of the AAAI Conference on Artificial Intelligence}, volume~38, pages 11443--11452, 2024.

\bibitem[Deng et~al.(2020)Deng, Kamani, and Mahdavi]{deng2020adaptive}
Yuyang Deng, Mohammad~Mahdi Kamani, and Mehrdad Mahdavi.
\newblock Adaptive personalized federated learning.
\newblock \emph{arXiv preprint arXiv:2003.13461}, 2020.

\bibitem[Eldele et~al.(2024)Eldele, Ragab, Chen, Wu, and Li]{eldele2024tslanet}
Emadeldeen Eldele, Mohamed Ragab, Zhenghua Chen, Min Wu, and Xiaoli Li.
\newblock Tslanet: Rethinking transformers for time series representation learning.
\newblock \emph{arXiv preprint arXiv:2404.08472}, 2024.

\bibitem[Fortuin et~al.(2018)Fortuin, H{\"u}ser, Locatello, Strathmann, and R{\"a}tsch]{fortuin2018som}
Vincent Fortuin, Matthias H{\"u}ser, Francesco Locatello, Heiko Strathmann, and Gunnar R{\"a}tsch.
\newblock Som-vae: Interpretable discrete representation learning on time series.
\newblock \emph{arXiv preprint arXiv:1806.02199}, 2018.

\bibitem[Fraikin et~al.(2023)Fraikin, Bennetot, and Allassonni{\`e}re]{fraikin2023t}
Archibald Fraikin, Adrien Bennetot, and St{\'e}phanie Allassonni{\`e}re.
\newblock T-rep: Representation learning for time series using time-embeddings.
\newblock \emph{arXiv preprint arXiv:2310.04486}, 2023.

\bibitem[Franceschi et~al.(2019)Franceschi, Dieuleveut, and Jaggi]{franceschi2019unsupervised}
Jean-Yves Franceschi, Aymeric Dieuleveut, and Martin Jaggi.
\newblock Unsupervised scalable representation learning for multivariate time series.
\newblock \emph{Advances in neural information processing systems}, 32, 2019.

\bibitem[Ghosh et~al.(2020)Ghosh, Chung, Yin, and Ramchandran]{ghosh2020efficient}
Avishek Ghosh, Jichan Chung, Dong Yin, and Kannan Ramchandran.
\newblock An efficient framework for clustered federated learning.
\newblock \emph{Advances in Neural Information Processing Systems}, 33:\penalty0 19586--19597, 2020.

\bibitem[Hochreiter(1997)]{hochreiter1997long}
S~Hochreiter.
\newblock Long short-term memory.
\newblock \emph{Neural Computation MIT-Press}, 1997.

\bibitem[Jin et~al.(2017)Jin, Ge, Netrapalli, Kakade, and Jordan]{jin2017escape}
Chi Jin, Rong Ge, Praneeth Netrapalli, Sham~M Kakade, and Michael~I Jordan.
\newblock How to escape saddle points efficiently.
\newblock In \emph{International conference on machine learning}, pages 1724--1732. PMLR, 2017.

\bibitem[Jin et~al.(2023)Jin, Wang, Ma, Chu, Zhang, Shi, Chen, Liang, Li, Pan, et~al.]{jin2023time}
Ming Jin, Shiyu Wang, Lintao Ma, Zhixuan Chu, James~Y Zhang, Xiaoming Shi, Pin-Yu Chen, Yuxuan Liang, Yuan-Fang Li, Shirui Pan, et~al.
\newblock Time-llm: Time series forecasting by reprogramming large language models.
\newblock \emph{arXiv preprint arXiv:2310.01728}, 2023.

\bibitem[Kim et~al.(2021)Kim, Al~Hakim, Haraldson, Eriksson, da~Silva, and Fischione]{kim2021dynamic}
Yeongwoo Kim, Ezeddin Al~Hakim, Johan Haraldson, Henrik Eriksson, Jos{\'e} Mairton~B da~Silva, and Carlo Fischione.
\newblock Dynamic clustering in federated learning.
\newblock In \emph{ICC 2021-IEEE International Conference on Communications}, pages 1--6. IEEE, 2021.

\bibitem[Krajzewicz et~al.(2012)Krajzewicz, Erdmann, Behrisch, and Bieker]{krajzewicz2012recent}
Daniel Krajzewicz, Jakob Erdmann, Michael Behrisch, and Laura Bieker.
\newblock Recent development and applications of sumo-simulation of urban mobility.
\newblock \emph{International journal on advances in systems and measurements}, 5\penalty0 (3\&4), 2012.

\bibitem[Kumar et~al.(2024)Kumar, Alawieh, Racz, Fakhreddine, and Millán]{10.1093/pnasnexus/pgae076}
Satyam Kumar, Hussein Alawieh, Frigyes~Samuel Racz, Rawan Fakhreddine, and José del~R Millán.
\newblock Transfer learning promotes acquisition of individual bci skills.
\newblock \emph{PNAS Nexus}, 3\penalty0 (2):\penalty0 pgae076, 02 2024.
\newblock ISSN 2752-6542.
\newblock \doi{10.1093/pnasnexus/pgae076}.
\newblock URL \url{https://doi.org/10.1093/pnasnexus/pgae076}.

\bibitem[Li et~al.(2021{\natexlab{a}})Li, Li, and Varshney]{li2021federated}
Chengxi Li, Gang Li, and Pramod~K Varshney.
\newblock Federated learning with soft clustering.
\newblock \emph{IEEE Internet of Things Journal}, 9\penalty0 (10):\penalty0 7773--7782, 2021{\natexlab{a}}.

\bibitem[Li et~al.(2020)Li, Sahu, Zaheer, Sanjabi, Talwalkar, and Smith]{li2020federated}
Tian Li, Anit~Kumar Sahu, Manzil Zaheer, Maziar Sanjabi, Ameet Talwalkar, and Virginia Smith.
\newblock Federated optimization in heterogeneous networks.
\newblock \emph{Proceedings of Machine learning and systems}, 2:\penalty0 429--450, 2020.

\bibitem[Li et~al.(2021{\natexlab{b}})Li, Hu, Beirami, and Smith]{li2021ditto}
Tian Li, Shengyuan Hu, Ahmad Beirami, and Virginia Smith.
\newblock Ditto: Fair and robust federated learning through personalization.
\newblock In \emph{International conference on machine learning}, pages 6357--6368. PMLR, 2021{\natexlab{b}}.

\bibitem[Mansour et~al.(2020)Mansour, Mohri, Ro, and Suresh]{mansour2020three}
Yishay Mansour, Mehryar Mohri, Jae Ro, and Ananda~Theertha Suresh.
\newblock Three approaches for personalization with applications to federated learning.
\newblock \emph{arXiv preprint arXiv:2002.10619}, 2020.

\bibitem[McMahan et~al.(2017)McMahan, Moore, Ramage, Hampson, and y~Arcas]{mcmahan2017communication}
Brendan McMahan, Eider Moore, Daniel Ramage, Seth Hampson, and Blaise~Aguera y~Arcas.
\newblock Communication-efficient learning of deep networks from decentralized data.
\newblock In \emph{Artificial intelligence and statistics}, pages 1273--1282. PMLR, 2017.

\bibitem[Mehta and Shao(2023)]{mehta2023greedy}
Manan Mehta and Chenhui Shao.
\newblock A greedy agglomerative framework for clustered federated learning.
\newblock \emph{IEEE Transactions on Industrial Informatics}, 19\penalty0 (12):\penalty0 11856--11867, 2023.

\bibitem[Nie et~al.(2022)Nie, Nguyen, Sinthong, and Kalagnanam]{nie2022time}
Yuqi Nie, Nam~H Nguyen, Phanwadee Sinthong, and Jayant Kalagnanam.
\newblock A time series is worth 64 words: Long-term forecasting with transformers.
\newblock \emph{arXiv preprint arXiv:2211.14730}, 2022.

\bibitem[Wang et~al.(2022)Wang, Zhang, Li, Tian, and Tedrake]{wang2022does}
Lirui Wang, Kaiqing Zhang, Yunzhu Li, Yonglong Tian, and Russ Tedrake.
\newblock Does learning from decentralized non-iid unlabeled data benefit from self supervision?
\newblock \emph{arXiv preprint arXiv:2210.10947}, 2022.

\bibitem[Wu et~al.(2022)Wu, Hu, Liu, Zhou, Wang, and Long]{wu2022timesnet}
Haixu Wu, Tengge Hu, Yong Liu, Hang Zhou, Jianmin Wang, and Mingsheng Long.
\newblock Timesnet: Temporal 2d-variation modeling for general time series analysis.
\newblock \emph{arXiv preprint arXiv:2210.02186}, 2022.

\bibitem[Xu et~al.(2014)Xu, Kliger, and Hero~III]{xu2014adaptive}
Kevin~S Xu, Mark Kliger, and Alfred~O Hero~III.
\newblock Adaptive evolutionary clustering.
\newblock \emph{Data Mining and Knowledge Discovery}, 28:\penalty0 304--336, 2014.

\bibitem[Zeng et~al.(2024)Zeng, Xiong, and Shi]{zeng2024metaclusterfl}
Hui Zeng, Shiyu Xiong, and Hongzhou Shi.
\newblock Metaclusterfl: Personalized federated learning on non-iid data with meta-learning and clustering.
\newblock In \emph{2024 International Joint Conference on Neural Networks (IJCNN)}, pages 1--10. IEEE, 2024.

\bibitem[Zhao et~al.(2018)Zhao, Li, Lai, Suda, Civin, and Chandra]{zhao2018federated}
Yue Zhao, Meng Li, Liangzhen Lai, Naveen Suda, Damon Civin, and Vikas Chandra.
\newblock Federated learning with non-iid data.
\newblock \emph{arXiv preprint arXiv:1806.00582}, 2018.

\end{thebibliography}
\bibliographystyle{plainnat}

\newpage
\appendix

\begin{center}
\Large
\textbf{Appendix}
\end{center}

This appendix is organized as follows:
\begin{itemize}
  \item \textbf{Section~1} describes the procedure for training task models using snapshot clustering.
  
  \item \textbf{Section~2} elaborates on the computation of the adaptive forgetting factor used in the evolutionary clustering algorithm.

  \item \textbf{Section~3} provides detailed proofs of the lemmas and theorem presented in the main text.

  \item \textbf{Section~4} outlines experimental implementation details.

  \item In the main paper, we assume the number of clusters is known a priori. \textbf{Section~5} explores strategies for estimating the number of clusters from task model output weights.

  \item \textbf{Section~6} demonstrates the compatibility of Fed-REACT with intermittent client participation.

  \item \textbf{Section~7} presents additional results on the RTD dataset under stationary cluster distributions. It also introduces \textit{Strategy 3}, where clients may migrate between clusters with high probability.

  \item For the SUMO dataset, the main text reports accuracy using fully personalized models, as the number of clusters is unknown. \textbf{Section~8} reports Fed-REACT performance under different assumed cluster counts.

  \item \textbf{Section~9} presents an ablation study on cluster heterogeneity and the number of clusters in the RTD dataset, including both stationary and non-stationary settings (Strategies 1 and 2).

  \item \textbf{Section~10} evaluates the performance of the time-smoothed gradient descent algorithm used in the representation learning phase.

  \item \textbf{Section~11} analyzes the computational and communication complexity of Fed-REACT compared to baseline methods.
\end{itemize}

\section{Task model training assisted by snapshot clustering}

Snapshot clustering groups clients based on the current weights of the task model / output layer, and then averages those weights to arrive at a cluster-specific task model. This procedure is formalized as Algorithm \ref{alg2} below. Note that snaphshot clustering may provide satisfactory performance when clients have exceedingly large number of labeled samples in training batches so that the models do not experience training variations. 

\setcounter{algorithm}{2}
\begin{algorithm}[!htbp]
\caption{Training of the task model assisted by snapshot clustering} \label{alg2}
\begin{algorithmic}[1]
\STATE \textbf{Initialize:} Global encoder parameters $\mathbf{\theta}_T$ obtained after $T$ rounds of federated representation learning presented in Alg. 1 in the main content

\FOR {client $k=1,2,..,K$}
\STATE Client $k$ trains the task model on the labeled local data. 
\STATE Client $k$ uploads the parameters $\mathbf{\theta}_{task}^k$ 
 of the task model to the server
\ENDFOR
\STATE Server clusters clients based on the weights of the task model $\{ \mathbf{\theta}_{task}^k \}_{k=1}^K$ and employs Agglomerative Hierarchical Clustering.
\FOR{cluster $c=1,2,..,C$ }
    \STATE Server aggregates the task model within cluster. Let $\mathcal{S}_t^c$ denote the set of clients in cluster $c$. Then

    \[
    \mathbf{\theta_{task}}^c = \sum_{k \in  \mathcal{S}_t^c} \frac{m_k}{M_c} \mathbf{\theta}_{task}^k
    \]
    where $m_k$ is the number of labeled samples on client $k$ and $M_c = \sum_{k \in  \mathcal{S}_t^c} m_k$

    \STATE Server transmits $\mathbf{\theta^c_{task}}$ to all clients $k \in \mathcal{S}_t^c $ 
\ENDFOR
\end{algorithmic}
\end{algorithm}

\section{Calculation of the forgetting factor $a_t$}

\begin{algorithm}[!htbp]
\caption{Estimating $a_t$ iteratively} \label{algAFFECT}
\begin{algorithmic}[1]

\FOR{iteration $iter=1,2,..,Max Iterations$  }
    \STATE Estimate $\mathcal{S}_t^c$ given $\hat{\psi}_{i,j,t-1}$, $\hat{a}_t$ which yield  $[\hat{\psi}_t]_{i, j}$. In our work, this is done via Agglomerative Hierarchical Clustering.   

    \STATE Compute $\hat{\mathbb{E}}[[W_t]_{i, j}]$ and $\hat{Var}([W_t]_{i, j})$ based on $\mathcal{S}_t^c$ as described above
    
    \STATE Estimate $\hat{a}_t$ using equation (\ref{AFFECT}).
    
\ENDFOR
\end{algorithmic}
\end{algorithm}

For completeness, we here summarize the derivation of the adaptive forgetting factor presented in \cite{xu2014adaptive}. Let $K$ denote the total number of clients, and let $L(a_t)$ be the Frobenius norm of the difference between the estimated and the true similarity matrix, i.e., 
\begin{align}
L(a_t)=\|\psi_{t}-a_t \hat{\psi}_{t-1} - (1-a_t)W_t \|_F^2
\end{align}

Then the risk function $R(a_t)=\mathbb{E}[L(a_t)]$ can be shown to take the form
\begin{align}
R(a_t)=\sum_{i=1}^{K}\sum_{j=1}^{K} \{(1-a_t)^2 Var([W_t]_{i, j})+a_t^2 ([\hat{\psi}_t]_{i, j}-[\psi_{t-1}]_{i, j})^2\},
\end{align}
where $[W_t]_{i, j}$, $[\hat{\psi}_t]_{i, j}$ and $[\psi_t]_{i, j}$ 
denote the entries at index $(i,j)$ of matrices $W_{t}$, $\hat{\psi}_{t}$ and $\psi_{t}$, respectively. To obtain this expression, it is assumed that $\mathbb{E}[[W_t]_{i, j}]=[\psi_t]_{i, j}$ and $Var([\psi_t]_{i, j})=0$. Taking the first derivative of $R(a_t)$ w.r.t to $a$ and setting it to zero yields
\begin{align}
\hat{a}_t=\frac{\sum_{i=1}^{K}\sum_{j=1}^{K} Var([W_t]_{i, j})}{\sum_{i=1}^{K}\sum_{j=1}^{K} ([\hat{\psi}_t]_{i, j}-[\psi_t]_{i, j})^2+Var([W_t]_{i, j})}. \label{AFFECT}
\end{align}

Note that the calculation in (\ref{AFFECT}) requires $\mathbb{E}[[W_t]_{i, j}]$ and $Var([W_t]_{i, j})$, which in turn requires knowledge of the clustering solution $\mathcal{S}_t^c$, which depends on $a_t$. \cite{xu2014adaptive} proposed to estimate $\mathbb{E}[[W_t]_{i, j}]$, $Var([W_t]_{i, j})$ and $a_t$ iteratively.  Suppose client $l$ is assigned to cluster $c$; then for $j \neq l$,
\begin{align}
\hat{\mathbb{E}}[[W_t]_{i, j}]=\sum_{i=l}\sum_{j\in c, j \neq l}\frac{1}{|c||c-1|} [W_t]_{i, j} \label{AFFECT2}
\end{align}
and
\begin{align}
\hat{\mathbb{E}}[[W_t]_{i, j}]=\sum_{i=1}^C\frac{1}{C} W_{i,i}. \label{AFFECT3}
\end{align}
For $k$ and $l$ in distinct clusters $c$ and $d$, respectively, it holds that
\begin{align}
\hat{\mathbb{E}}[[W_t]_{k, l}]=\sum_{i\in c}\sum_{j\in d}\frac{1}{|c||d|} [W_t]_{i, j}. \label{AFFECT4}
\end{align}
Estimates of the variances can be computed in a similar manner and are thus omitted for the sake of brevity. The resulting procedure is formalized as Algorithm \ref{algAFFECT}. In our simulations, we set the number of iterations to $5$.

\section{Proof of Theorem}
Recall the assumption in the main paper, 

{\bf Assumption 3.1.} (a) Loss function $f_{t, i}$ is bounded above by $M$ for all clients $i$ and times $t$. (b) Loss function $f_{t, i}$ is $L$-Lipschitz and $\beta$-smooth. (c) The stochastic gradient $\Tilde{\nabla} f(\cdot)$ is unbiased and its standard deviation 
is bounded above by $\sigma$. The error between the projected stochastic gradient $Proj \Tilde{\nabla} f(\cdot)$ and the stochastic gradient $ \Tilde{\nabla} f(\cdot)$ is $\epsilon_{proj} = Proj (\Tilde{\nabla} f(\cdot)) -\Tilde{\nabla} f(\cdot) $ with $\| \epsilon_{proj} \|^2 \leq \epsilon^2$.

and the defined local regret at client $k$ and the global regret as
\[
S_{t, w, \gamma, k}(\theta_t) = \frac{1}{W}\sum_{j=0}^{w-1} \gamma^j f_{t-j, k}(\theta_{t-j}), \;\; 
S_{t, w, \gamma}(\theta_t) = \frac{1}{K}\sum_{k=1}^K \frac{1}{W}\sum_{j=0}^{w-1} \gamma^j f_{t-j, k}(\theta_{t-j}),
\]
respectively.

With the assumption, we first obtain the following lemmas:

\begin{lemma}
Suppose all of the above assumptions are satisfied. Then for any $\gamma \in (0, 1)$, $\beta$ and $\eta$, it holds that
\begin{align*}
    (\frac{\eta}{4} - \frac{\eta^2\beta}{8}) \| \nabla S_{t, w, \gamma}(\theta_t) \|^2  &  \leq S_{t, w, \gamma}(\theta_{t}) - S_{t+1, w, \gamma}(\theta_{t+1}) +  S_{t+1, w, \gamma}(\theta_{t+1}) - S_{t, w, \gamma}(\theta_{t+1}) \\ &  + \eta^2 \frac{\beta}{4}\frac{\sigma^2 (1-\gamma^{2w})}{W^2(1-\gamma^2)}  \nonumber + (\frac{\eta}{4} + \frac{3\eta^2\beta}{8} )\epsilon^2.
\end{align*}
\end{lemma}

\begin{lemma}
Suppose all of the above assumptions are satisfied. Then for any $\gamma \in (0, 1)$ and $w$, it holds that
\begin{align*}
    S_{t+1, w, \gamma}(\theta_{t+1}) - S_{t, w, \gamma}(\theta_{t+1}) \leq \frac{M(1+\gamma^{w-1})}{W} \nonumber + \frac{M(1-\gamma^{w-1})(1+ \gamma)}{W(1-\gamma)}.
\end{align*}
\end{lemma}

\begin{lemma}
Suppose all of the above assumptions are satisfied. Then for any $\gamma \in (0, 1)$ and $w$, it holds that
\begin{align*}
    S_{t, w, \gamma}(\theta_{t}) - S_{t+1, w, \gamma}(\theta_{t+1}) &  \leq \frac{2M(1-\gamma^w)}{W(1-\gamma)}.
\end{align*}
\end{lemma}

\begin{proof}

Using $\beta$-smoothness assumption of $f_{t, k}$ functions, it can be shown that $S_t$ is $\beta$-smooth. Then we have
\begin{align*}
    & S_{t, w, \gamma}(\theta_{t+1}) - S_{t, w, \gamma}(\theta_t)  = \frac{1}{K}\sum_{k=1}^K S_{t, w, \gamma, k}(\theta_{t+1}) - S_{t, w, \gamma, k}(\theta_t) \\
    & \leq \frac{1}{K}\sum_{k=1}^K \langle \nabla S_{t, w, \gamma, k}(\theta_t), \theta_{t+1} - \theta_t \rangle + \frac{\beta}{2} \|\theta_{t+1} - \theta_t \|^2 \\
    & = \langle \nabla S_{t, w, \gamma}(\theta_t), \theta_{t+1} - \theta_t \rangle + \frac{\beta}{2} \|\theta_{t+1} - \theta_t \|^2 \\
    & = -\frac{\eta}{2} \langle \nabla S_{t, w, \gamma}(\theta_t),  \Tilde{\nabla} S_{t, w, \gamma}(\theta_t) + \epsilon_{proj} \rangle -\frac{\eta}{2} \langle \nabla S_{t, w, \gamma}(\theta_t),  \Tilde{\nabla} S_{t, w, \gamma}(\theta_t) + \epsilon_{proj} -\nabla S_{t, w, \gamma}(\theta_t) \rangle \\
    & - \frac{\eta}{2} \|\nabla S_{t, w, \gamma}(\theta_t) \|^2 + \frac{\eta^2 \beta}{4} \|\Tilde{\nabla} S_{t, w, \gamma}(\theta_t) + \epsilon_{proj} -  \nabla S_{t, w, \gamma}(\theta_t) + \nabla S_{t, w, \gamma}(\theta_t) \|^2 \\
    & + \frac{\eta^2\beta}{4} \|\Tilde{\nabla} S_{t, w, \gamma}(\theta_t) + \epsilon_{proj} \|^2
\end{align*}
where $\epsilon_{proj}$ represents the projection error.

Therefore,
\begin{align*}
    & S_{t, w, \gamma}(\theta_{t+1}) - S_{t, w, \gamma}(\theta_t) \\
    & \leq -(\frac{\eta}{2} - \frac{\eta^2\beta}{4}) \| \nabla S_{t, w, \gamma}(\theta_t) \|^2 - (\frac{\eta}{2} - \frac{\eta^2\beta}{4}) \langle \nabla S_{t, w, \gamma}(\theta_t),  \Tilde{\nabla} S_{t, w, \gamma}(\theta_t) + \epsilon_{proj} - \nabla S_{t, w, \gamma}(\theta_t) \rangle \\
    & + \frac{\eta^2 \beta}{4} \|\Tilde{\nabla} S_{t, w, \gamma}(\theta_t) +\epsilon_{proj} -  \nabla S_{t, w, \gamma}(\theta_t) \|^2 \\
    & \leq -(\frac{\eta}{2} - \frac{\eta^2\beta}{4}) \| \nabla S_{t, w, \gamma}(\theta_t) \|^2 - (\frac{\eta}{2} - \frac{\eta^2\beta}{4}) \langle \nabla S_{t, w, \gamma}(\theta_t),  \Tilde{\nabla} S_{t, w, \gamma}(\theta_t)  - \nabla S_{t, w, \gamma}(\theta_t) \rangle \\
    & -  (\frac{\eta}{2} - \frac{\eta^2\beta}{4}) \langle \nabla S_{t, w, \gamma}(\theta_t), \epsilon_{proj} \rangle + \frac{\eta^2\beta}{2} \|\Tilde{\nabla} S_{t, w, \gamma}(\theta_t) -  \nabla S_{t, w, \gamma}(\theta_t) \|^2 + \frac{\eta^2\beta \epsilon^2}{2} \\
    & \leq -\frac{1}{2}(\frac{\eta}{2} - \frac{\eta^2\beta}{4}) \| \nabla S_{t, w, \gamma}(\theta_t) \|^2 - (\frac{\eta}{2} - \frac{\eta^2\beta}{4}) \langle \nabla S_{t, w, \gamma}(\theta_t),  \Tilde{\nabla} S_{t, w, \gamma}(\theta_t)  - \nabla S_{t, w, \gamma}(\theta_t) \rangle \\
    & + \frac{1}{2}(\frac{\eta}{2} - \frac{\eta^2\beta}{4}) \epsilon^2 + \frac{\eta^2\beta}{2} \|\Tilde{\nabla} S_{t, w, \gamma}(\theta_t) -  \nabla S_{t, w, \gamma}(\theta_t) \|^2 + \frac{\eta^2\beta \epsilon^2}{2}.
\end{align*}

By applying the conditional expectation $\mathbb{E}[\cdot | \theta_t]$ to both sides of the inequality, we obtain
\begin{align*}
    & (\frac{\eta}{4} - \frac{\eta^2\beta}{8}) \| \nabla S_{t, w, \gamma}(\theta_t) \|^2 \\
    & \leq \mathbb{E}[S_{t, w, \gamma}(\theta_{t}) - S_{t, w, \gamma}(\theta_{t+1}) ] + \eta^2 \frac{\beta}{2}\frac{\sigma^2 (1-\gamma^{2w})}{W^2(1-\gamma^2)} + (\frac{\eta}{4} - \frac{\eta^2\beta}{8} + \frac{\eta^2\beta}{2})\epsilon^2 \\
    & = S_{t, w, \gamma}(\theta_{t}) - S_{t+1, w, \gamma}(\theta_{t+1}) +  S_{t+1, w, \gamma}(\theta_{t+1}) - S_{t, w, \gamma}(\theta_{t+1}) + \eta^2 \frac{\beta}{4}\frac{\sigma^2 (1-\gamma^{2w})}{W^2(1-\gamma^2)} \\
    & + (\frac{\eta}{4} - \frac{\eta^2\beta}{8} + \frac{\eta^2\beta}{2})\epsilon^2   \\
    & = S_{t, w, \gamma}(\theta_{t}) - S_{t+1, w, \gamma}(\theta_{t+1}) +  S_{t+1, w, \gamma}(\theta_{t+1}) - S_{t, w, \gamma}(\theta_{t+1}) + \eta^2 \frac{\beta}{4}\frac{\sigma^2 (1-\gamma^{2w})}{W^2(1-\gamma^2)} \\
    & + (\frac{\eta}{4} + \frac{3\eta^2\beta}{8} )\epsilon^2 . 
\end{align*}
Rearranging the left and right side terms gives the inequality in the first Lemma.

Next, we derive the upper bounds for $S_{t+1, w, \gamma}(\theta_{t+1}) - S_{t, w, \gamma}(\theta_{t+1}) $ and $ S_{t, w, \gamma}(\theta_{t}) - S_{t+1, w, \gamma}(\theta_{t+1}) $. Recall that each loss function $f_t$ is upper bounded by $M$, i.e., $|f_t(x)| \leq M$. Then

\begin{align*}
    S_{t+1, w, \gamma}(\theta_{t+1}) - S_{t, w, \gamma}(\theta_{t+1}) & = \frac{1}{W}\sum_{j=0}^{w-1} \gamma_j (f_{t+1-j}(\theta_{t+1-j}) - f_{t-j}(\theta_{t+1-j}) ) \\
    & = \frac{1}{W}[f_{t+1}(\theta_{t+1}) - f_t(\theta_{t+1}) + \gamma f_t(\theta_t) - \gamma f_{t-1}(\theta_t) + \cdots \\
    & + \gamma^{w-1}f_{t-w+2}(\theta_{t-w+2}) - \gamma^{w-1}f_{t-w+1}(\theta_{t-w+2}) ] \\
    & \leq \frac{M(1+\gamma^{w-1})}{W} + \frac{M(1-\gamma^{w-1})(1+ \gamma)}{W(1-\gamma)}
\end{align*}

\begin{align*}
    S_{t, w, \gamma}(\theta_{t}) - S_{t+1, w, \gamma}(\theta_{t+1}) & = \frac{1}{W}\sum_{j=0}^{w-1} \gamma^j (f_{t-j}(\theta_{t-j}) - f_{t+1-j}(\theta_{t+1-j}) ) \\
    & \leq \frac{2M(1-\gamma^w)}{W(1-\gamma)}
\end{align*}
This completes the proof of Lemma~2 and 3.
\end{proof}

We now proceed with the proof of the main theorem based on the established inequalities:
\begin{proof}
Using the inequalities above, we derive an upper bound on $\|  \nabla S_{t, w, \gamma}(\theta_t) \|^2 $ as
\begin{align*}
    & \|  \nabla S_{t, w, \gamma}(\theta_t) \|^2 \\
    & \leq \frac{  \frac{2M(1-\gamma^w)}{W(1-\gamma)} + \frac{M(1+\gamma^{w-1})}{W} + \frac{M(1-\gamma^{w-1})(1+ \gamma)}{W(1-\gamma)} + \eta^2 \frac{\beta}{4}\frac{\sigma^2 (1-\gamma^{2w})}{W^2(1-\gamma^2)} + (\frac{\eta}{4} - \frac{\eta^2\beta}{8} + \frac{\eta^2\beta}{2})\epsilon^2 }{(\frac{\eta}{4} - \frac{\eta^2\beta}{8}) }.
\end{align*}

Substituting $\eta = \frac{1}{\beta}$ yields
\begin{align*}
    & \|  \nabla S_{t, w, \gamma}(\theta_t) \|^2 \\
    & \leq \frac{8\beta M}{W} (\frac{2(1-\gamma^w)}{1-\gamma} + (1+\gamma^{w-1}) + \frac{(1-\gamma^{w-1})(1+ \gamma)}{1-\gamma} ) + \frac{2\sigma^2 (1-\gamma^{2w})}{W^2(1-\gamma^2)} + \frac{5}{8}\epsilon^2 \\
    & \leq \frac{8\beta M}{W} (\frac{2(1-\gamma^w)}{1-\gamma} + (1+\gamma^{w-1}) + \frac{(1-\gamma^{w})(1+ \gamma)}{1-\gamma} ) + \frac{2\sigma^2 (1-\gamma^{2w})}{W^2(1-\gamma^2)} + \frac{5}{8}\epsilon^2 \\
    & = \frac{8\beta M}{W} (\frac{(1-\gamma^{w})(3+ \gamma)}{1-\gamma} + (1+\gamma^{w-1}) ) + \frac{2\sigma^2 (1-\gamma^{2w})}{W^2(1-\gamma^2)} + \frac{5}{8}\epsilon^2 \\
    & \leq \frac{8\beta M}{W}(4\frac{(1-\gamma^{w})}{1-\gamma} + \frac{1+\gamma^{w-1}}{1-\gamma} ) + \frac{2\sigma^2 (1-\gamma^{2w})}{W^2(1-\gamma^2)} + \frac{5}{8}\epsilon^2 \\
    & \leq \frac{32\beta M}{W} (\frac{2-\gamma^w + \gamma^{w-1}}{1-\gamma} ) + \frac{2\sigma^2 (1-\gamma^{2w})}{W^2(1-\gamma^2)} + \frac{5}{8}\epsilon^2.
\end{align*}
 When $ \gamma \to 1^{-}$,
\begin{align*}
    \lim_{\gamma \to 1^{-}} \|  \nabla S_{t, w, \gamma}(\theta_t) \|^2 \leq \frac{1}{W}(64\beta M +2\sigma^2) + \frac{5}{8}\epsilon^2.
\end{align*}
Telescoping $t$ from $1$ to $T$, we obtain
\begin{align*}
    \lim_{\gamma \to 1^{-}} \sum_{t=1}^T  \| \nabla S_{t, w, \gamma}(\theta_t) \|^2 \leq \frac{T}{W} (64\beta M + 2\sigma^2) + \frac{5}{8}\epsilon^2T
\end{align*}
and
\begin{align*}
    \lim_{\gamma \to 1^{-}} \frac{1}{T} \sum_{t=1}^T  \| \nabla S_{t, w, \gamma}(\theta_t) \|^2 \leq \frac{1}{W} (64\beta M + 2\sigma^2) + \frac{5}{8}\epsilon^2
\end{align*}
This concludes the proof of the Theorem.
\end{proof}

\section{Experiment implementation details}
In this section, we provide details of the experimental settings leading to the results presented in the main paper. As one of the benchmarking algorithms, a single-layer LSTM model is used with a feature embedding dimension $128$ and hidden size $256$. In the TimesNet model, the number of layers is set to $2$, the number of kernels equal to $6$ and the feed-forward dimension equal to $100$. For PatchTST, the patch size is $10$ with equal stride; the number of transform layers is equal to $3$, with model dimension equal to $256$ and $8$ heads. The feed-forward dimension for the PatchTST is equal to $512$. Each client performs local supervised training for $100$ epochs with a batch size of $50$, using the Adam optimizer with a learning rate of $ 0.001$. A total of $10$ communication rounds are conducted, with model aggregation performed at the server.

Regarding the implementation of Fed-REACT algorithm, the encoder uses causal time dilated CNNs consisting of $10$ 1S convolutional blocks, with dilation increasing by a factor of $2$ in each layer. Each block uses leaky ReLU activation (negative slope $0.01$), followed by a linear layer that outputs features of size $320$. The encoder is trained using contrastive loss as outlined in \cite{franceschi2019unsupervised}. The task model is an SVM classifier that predicts one out of ten classes based on the encoded features. Each client performs $500$ training steps per communication round, with a batch size of $10$, using the Adam optimizer with learning rate $0.001$.

To create heterogeneous clusters in Section 4.1 of the main text and section 7 of the supplementary text, we use Dirichlet sampling to distribution examples for each label among the three clusters. For the 10 and 50 client settings in Section 4.1 of the main text, we set the parameter of the Dirichlet sampling, $\alpha$, to $0.1$. For the 10 client setting, this yields the distribution presented in Fig. \ref{fig:sample_dist}. For the $100$ client setting in Section 7, $\alpha$ is set to $2.5$.

\begin{figure}[!htbp]
    \centering
    \includegraphics[width=0.65\linewidth]{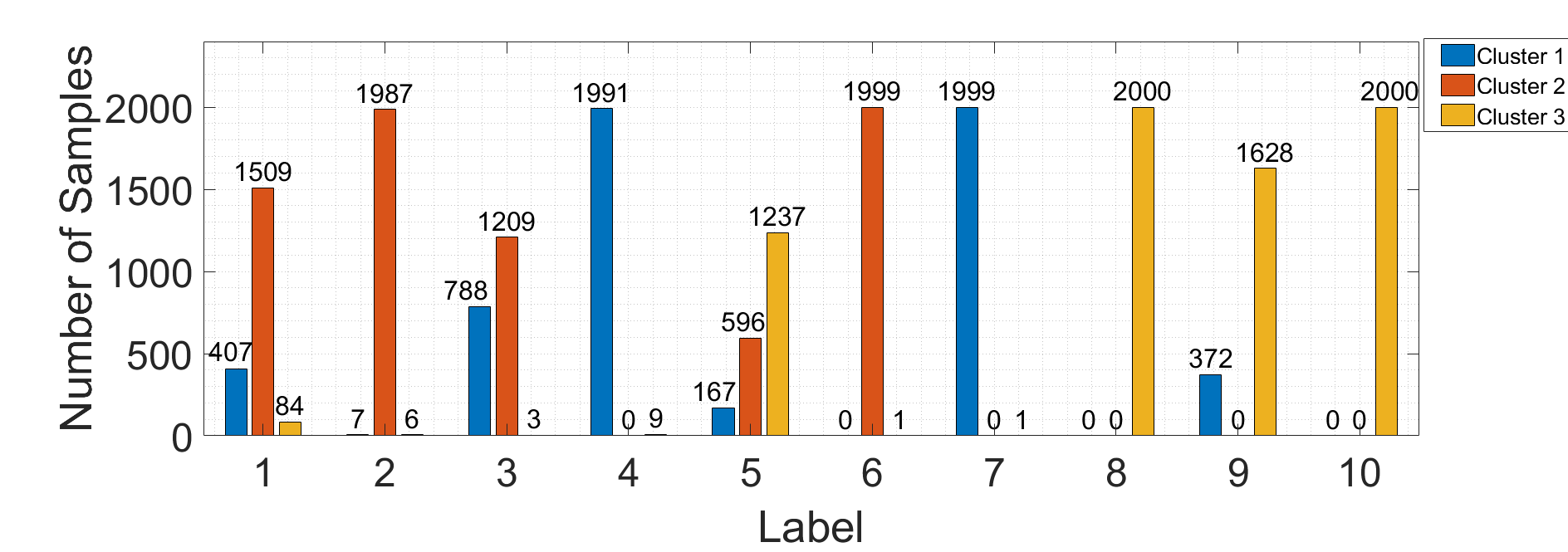} 
    \caption{Label distribution for three clusters generated using Dirichlet distribution with $\alpha=0.1$. Cluster 1 is primarily composed of digits 3 and 6, Cluster 2 contains digits 0, 1, 2, and 5, while Cluster 3 consists of digits 4, 7, 8, and 9.}
    \label{fig:sample_dist}
\end{figure}

\section{Estimation of the number of clusters}

In the main text of the paper, we have assumed that we know the number of clusters apriori. 
In this section, we explore two strategies for estimating the number of clusters that can be used: namely the elbow method and the Sillhouette score.

\begin{figure*}[!htbp]
\centering
  \begin{minipage}[t]{.45\linewidth}
    \includegraphics[width=\linewidth]{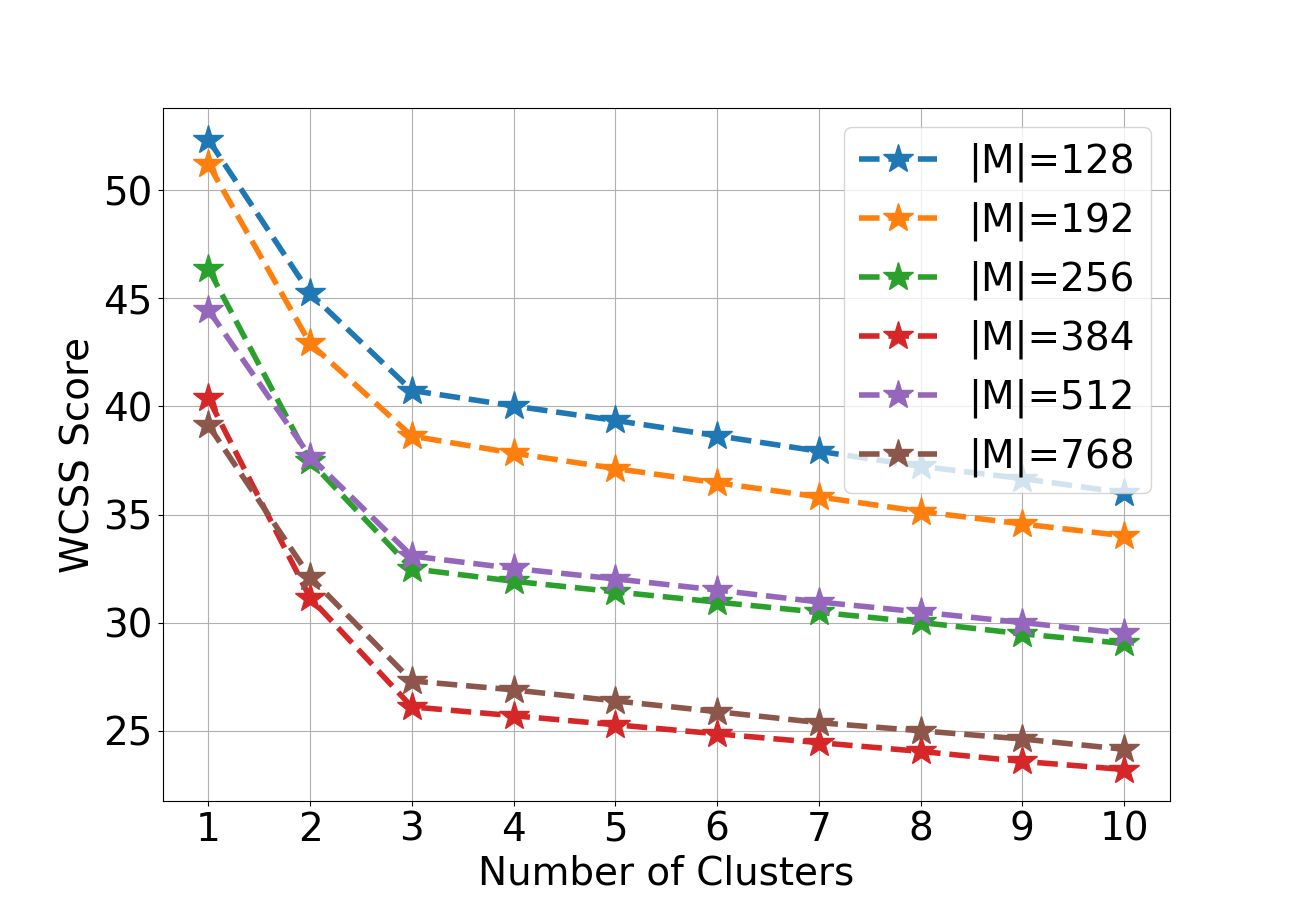}%
    \caption{WCSS scores against the number of clusters for various sizes of the local datasets $|\mathcal{M}_t^k|$}
    \label{fig:clusternumest1}
  \end{minipage}\hfil
  \begin{minipage}[t]{.45\linewidth}
    \includegraphics[width=\linewidth]{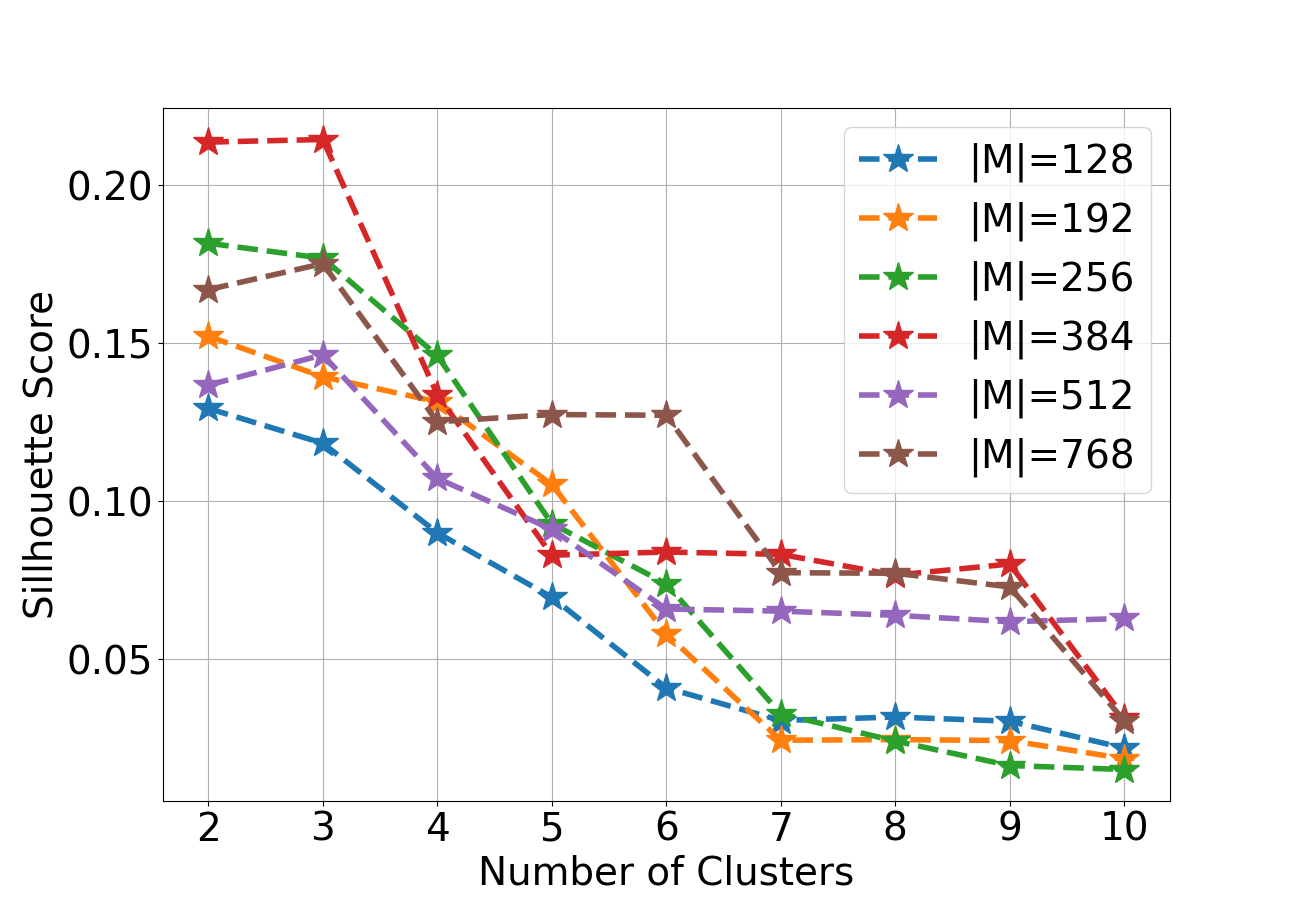}%
    \caption{Sillhouette scores against the number of clusters for various sizes of the local datasets $|\mathcal{M}_t^k|$}
    \label{fig:clusternumest2}
  \end{minipage}
\end{figure*}



\begin{enumerate}
\item \textbf{Elbow Method}: In this method, we compute the within-cluster sum of squares (WCSS) for clustering solutions with varying numbers of clusters. WCSS measures cluster compactness by summing the squared distances between each client and its assigned cluster center. Typically, WCSS decreases as the number of clusters grows. We select the optimal cluster count using the "elbow" method, identifying the point at which adding more clusters no longer significantly reduces the WCSS.    

\item \textbf{Silhouette Score}: The Silhouette score, ranging between $-1$ and $1$, measures how closely each client matches its own cluster compared to neighboring clusters. We select the optimal number of clusters as the one that maximizes the Silhouette score.
\end{enumerate}

Results are presented in Figures~\ref{fig:clusternumest1} and~\ref{fig:clusternumest2} for the default scenario of 100 clients partitioned into 3 clusters (with $\alpha=0.1$) under varying local dataset sizes $|\mathcal{M}_t^k|$. As evident from the plots, the elbow method consistently identifies the correct cluster count (3 clusters). However, the Silhouette score correctly identifies the optimal cluster number only when local datasets are sufficiently large.

\section{Fed-REACT with
intermittent client participation}

\begin{figure*}[!htbp]
\centering
  \begin{minipage}[t]{.45\linewidth}
    \includegraphics[width=\linewidth]{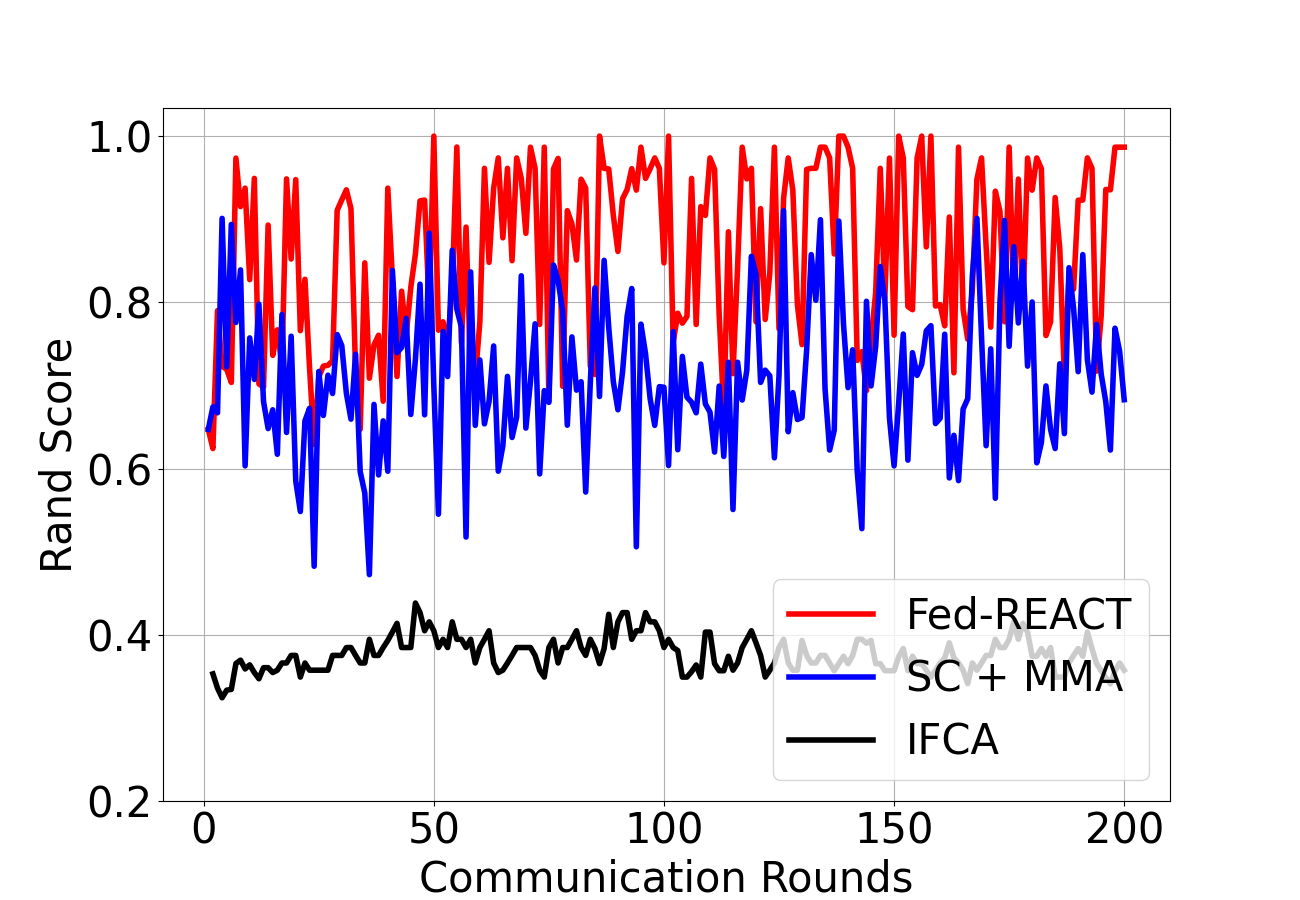}%
    \caption{\small Rand score against the ground truth 
 for different method on the RTD dataset for intermittent client participation (Client Participation Ratio: 0.33)}%
    \label{fig:Async}
  \end{minipage}\hfil
  \begin{minipage}[t]{.45\linewidth}
    \includegraphics[width=\linewidth]{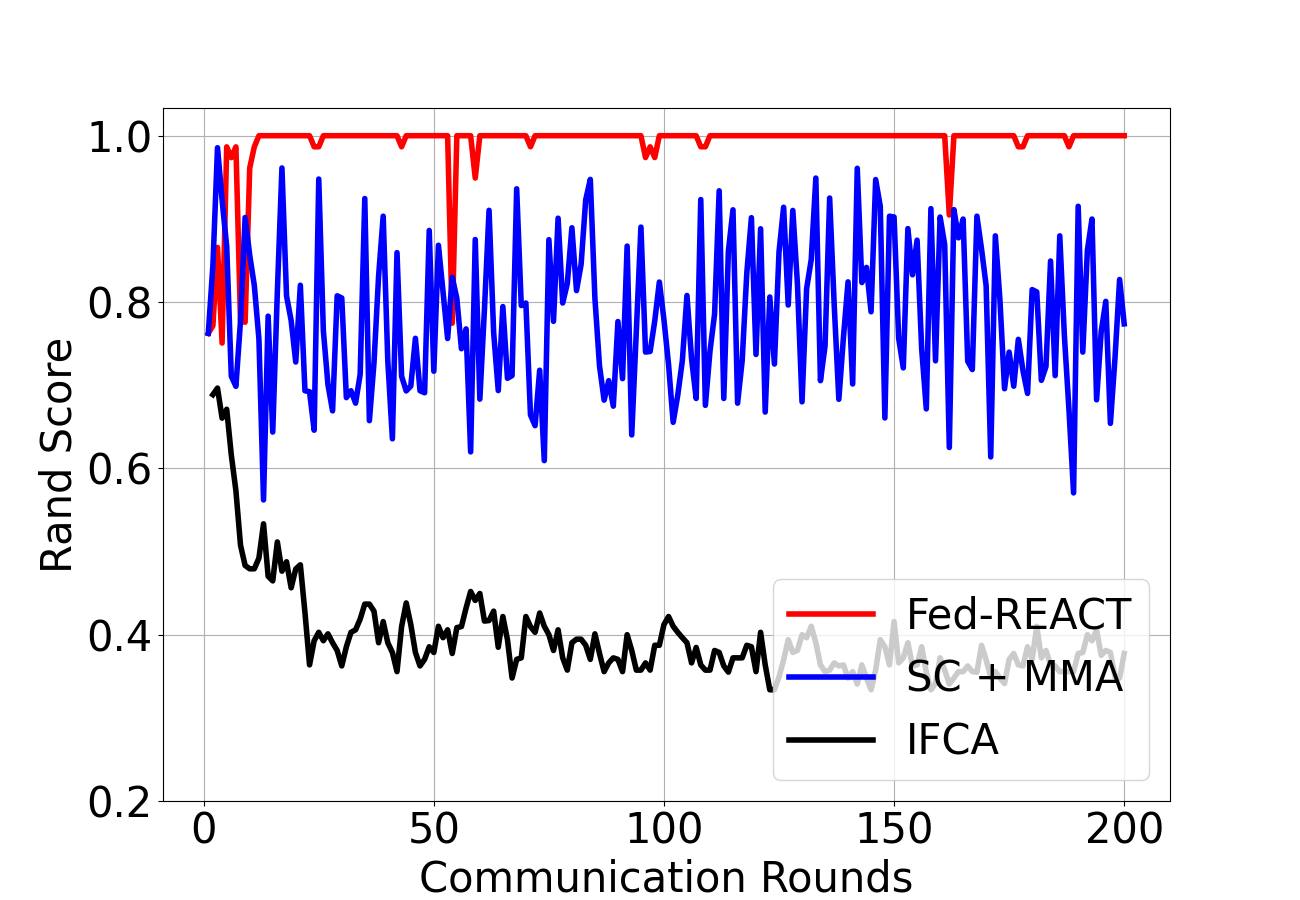}%
    \caption{\small Rand score against the ground truth 
 for different method on the RTD dataset for intermittent client participation (Client Participation Ratio: 0.50)}
    \label{fig:Async0pt5}
  \end{minipage}\hfill
\end{figure*}

Fed-REACT can be extended to settings with partial client participation through minor modifications. To compute the similarity matrix under partial client participation, we use the most recently saved output layer parameters for clients that do not participate in a given round. Only the participating clients are used to update the cluster-specific task models. We evaluate this asynchronous setting on the RTD dataset using a 100-client setup with stationary clusters, as described in Section 7 of the supplementary material with the exception that we set the Dirichlet sampling parameter to $\alpha=1.75$. For each of the three clusters, we randomly choose one-third of the clients at a given round. We also explore the case when the client participation is $50 \%$.  The rand scores for the clustering solutions are plotted in Fig. \ref{fig:Async} and \ref{fig:Async0pt5} with the corresponding accuracies in Table. \ref{tab:async}, showing Fed-REACT outperforming clustering baseline schemes in the scenario with intermittent client participation.

\begin{table*}[!htbp]
  \centering
  \small
  \resizebox{0.75\linewidth}{!}{
  \begin{tabular}{c|c|c|c|c}\toprule
    Client Participation Ratio & Fed - REACT w/ A1 & Fed - REACT w/ A2 & IFCA & FLSC  \\ \midrule
      0.33 & \textbf{0.761} & 0.753 & 0.719 & 0.725 \\ \midrule
      0.50 & \textbf{0.768} & 0.756 & 0.721 & 0.730 \\ \bottomrule
  \end{tabular}
  }
  \caption{Results for Fed-REACT vs. baselines for the setting where the fraction of participating clients is 0.33 and 0.5}
  \label{tab:async}
\end{table*}

\section{Additional experiments on the RTD dataset}

\subsection{Evolutionary clustering on stationary distribution}

In the next set of experiments, we evaluate impact of the clustering method utilized in the second phase of the Fed-REACT algorithm. The considered baseline clustered FL methods include IFCA (\cite{ghosh2020efficient}), 
FL with Soft Clustering (FLSC) (\cite{li2021federated}) and FLACC (\cite{mehta2023greedy}). 



To generate clusters for the stationary setting, we partition the RTD dataset using Dirichlet Sampling with $\alpha=2.5$ for the $10$ client setting, and $\alpha =1.5$ for the $100$ client setting. The client datasets are then uniformly sampled from their respective clusters. At time $t$, client $k$ trains the task models using $|\mathcal{M}^k_t|=64$ labeled samples, emulating the setting where the number of labeled samples is rather limited; a total of $60$ communication rounds is conducted for the $10$ clients setting and $200$ for the $100$ clients setting. Figure \ref{fig:2rand} and \ref{fig:3rand} show the progression of the Rand score through the communication rounds for $10$ and $100$ clients, respectively. In the latter case, Clusters 1, 2 and 3 contain $33$, $33$ and $34$ clients, respectively. Figure \ref{fig:2rand} demonstrates that Fed-REACT's evolutionary clustering technique correctly groups the clients in as few as 3 communication rounds, while the snapshot clustering methods struggle to discover the ground truth. 
Even when the number of clients in the system increases to $100$, the observed Rand score of Fed-REACT's evolutionary clustering method rapidly identifies true clusters and steadily maintains the correct solution, while the competing methods suffer from oscillations in the cluster membership and generally fail to approach the ground truth.
Accuracies for different methods in a system with $10$ and $100$ clients are reported in Table~ \ref{tab:cluster_accstaionary}. Specifically, for each algorithm we calculate the instantaneous accuracy averaged over all communication rounds. The results show that by including historical information, evolutionary clustering methods are capable of discovering the true structure and memberships of clusters, and generally lead to task models that achieve higher accuracy than the schemes ignoring past information. The most accurate performance is achieved by Fed-REACT that relies on approach A2 for task model aggregation.

\begin{figure*}[!htbp]
\centering
  \begin{minipage}[t]{.45\linewidth}
    \includegraphics[width=\linewidth]{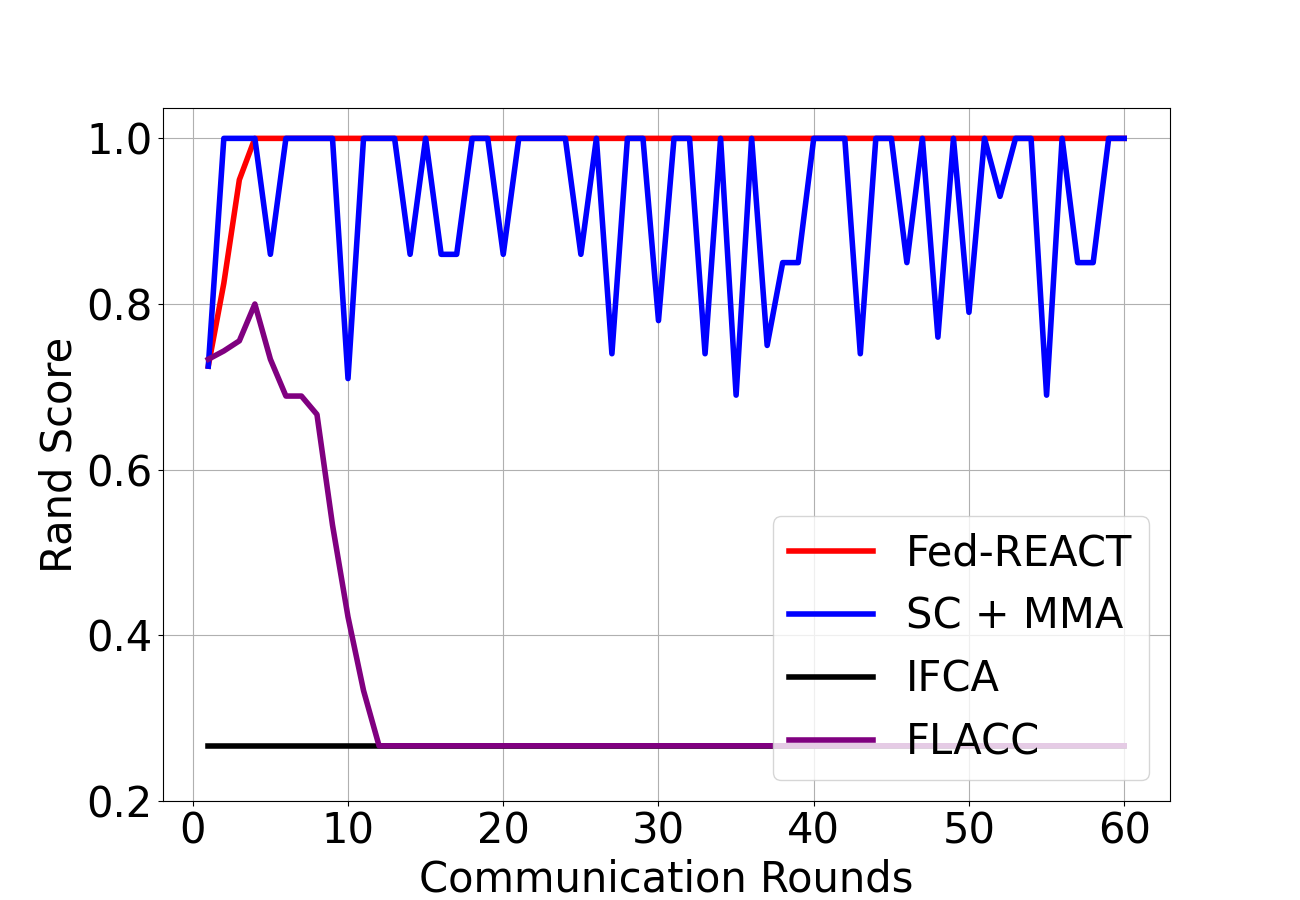 }%
    \caption{\small Rand score against the ground truth 
 for different method on the RTD dataset for stationary cluster distributions for a setup with 10 clients}%
    \label{fig:2rand}
  \end{minipage}\hfil
  \begin{minipage}[t]{.45\linewidth}
    \includegraphics[width=\linewidth]{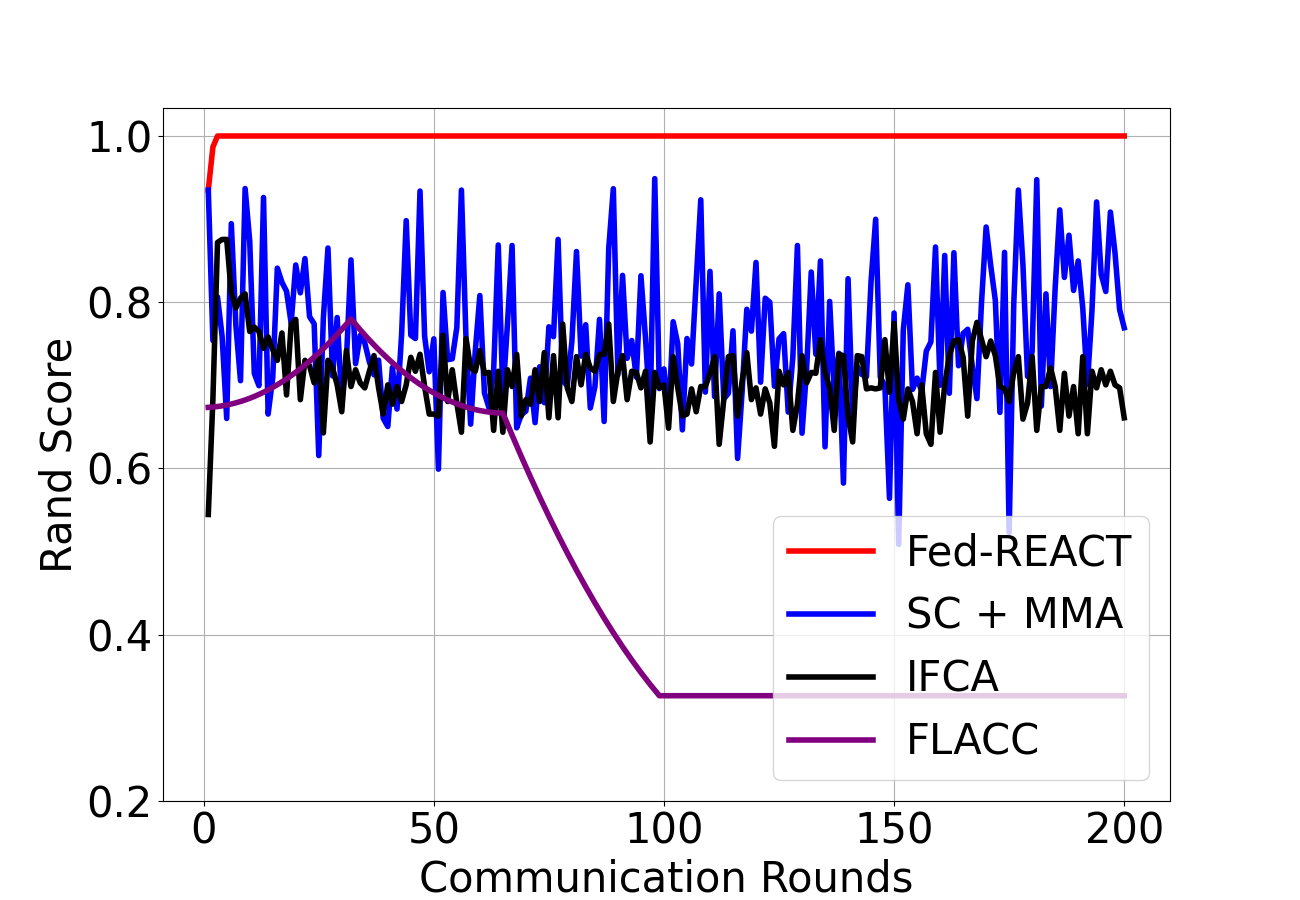}%
    \caption{\small \small Rand score against the ground truth 
 for different method on the RTD dataset for stationary cluster distributions for a setup with 100 clients}%
    \label{fig:3rand}
  \end{minipage}\hfill
\end{figure*}

\begin{table*}[t]
  \centering
  \small
  \resizebox{0.95\linewidth}{!}{
  \begin{tabular}{c|c|c|c|c|c|c|c}\toprule
    Dataset & Fed - REACT w/ A1 & Fed - REACT w/ A2 &  SC + MMA & EC + MMA & IFCA & FLSC &  FLACC \\ \midrule
    RTD - 10 clients (Stationary) & 0.909 & \textbf{0.928} & 0.763 & 0.859 & 0.774 & 0.830 & 0.755  \\ \midrule
    RTD - 100 clients (Stationary) & 0.750 & \textbf{0.751} & 0.716 & 0.737 & 0.739 & 0.740 &0.729  \\ \midrule
    RTD - 100 clients (Non-Stationary: Strategy 3) & 0.857 &\textbf{0.861} & 0.760 & 0.850 & 0.798 & 0.582 & 0.428 \\ \bottomrule
  \end{tabular}
  }
  \caption{The test accuracy computed after $T_{task}$ rounds of Fed-REACT vs. baselines on RTD dataset for the Stationary Setting with 10 and 100 clients. The accuracy is computed by averaging cluster-specific model accuracies defined as $\frac{1}{K}\sum_{C_i}{\sum_{k \in C_i}}\mathrm{Acc}_{C_i}(\mathcal{D}_{k, test}) $, where $K$ is the number of clients and $\mathrm{Acc}_{C_i}(\mathcal{D}_{k, test})$ denotes the accuracy of the model for cluster $C_i$ tested on the dataset that belongs to client $k \in C_i$.}
  \label{tab:cluster_accstaionary}
\end{table*}

\subsection{Evolutionary clustering for non-stationary distribution with client migration}

In strategy 2 discussed in the main text, clients randomly (with a small probability) move to a different cluster temporarily, and return to their native clusters with a high probability. In this section, we explore a more challenging setting, i.e., \textit{strategy 3}, in which clients \textit{migrate} to a different cluster with a small probability (p=0.005 for our experiments). The cluster distributions, however, are generated in the same fashion as strategy 2, and other experimental settings are kept same as strategy 2 as well. We plot the rand scores in Fig. \ref{fig:migrate} and record the accuracies in Table \ref{tab:cluster_accstaionary}. As the results presented show, Fed-REACT is able to identify the clustering solution correctly in this more challenging setting, which translates to an advantage in terms of accuracy over the baseline schemes. While IFCA is eventually able to identify the correct clustering solution, the random initialization of the cluster models required by the algorithm results in inferior performance in terms of accuracy. 

\begin{figure}[htbp]
    \centering
    \includegraphics[width=0.45\linewidth]{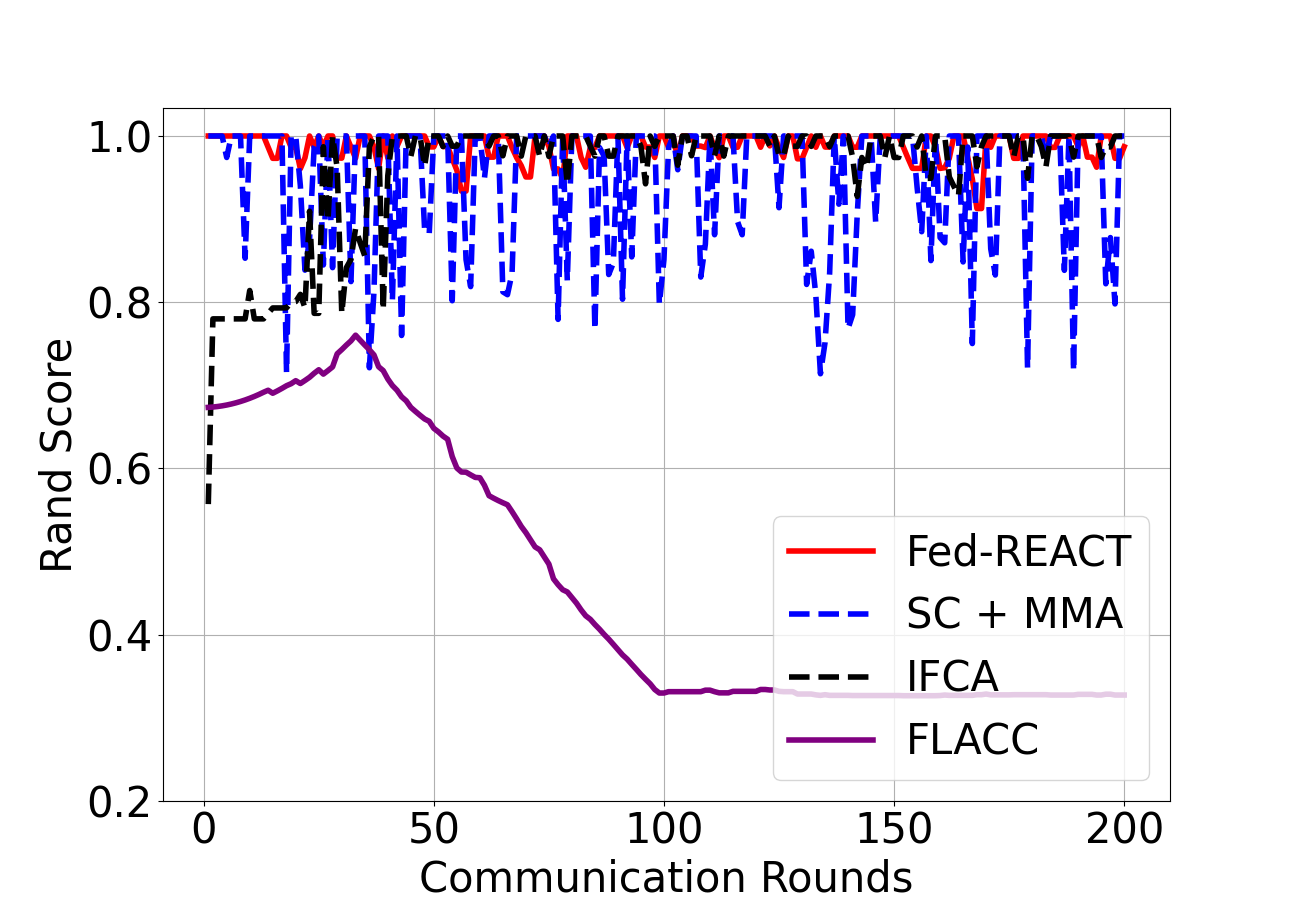 }
    \caption{Rand score against the ground truth 
 for different method on the RTD dataset with client migration setting}
    \label{fig:migrate}
\end{figure}

\section{Experiments for the SUMO dataset in a clustered setting}

\begin{table}
\centering
  \resizebox{\linewidth}{!}{%
  \begin{tabular}{c|cccccc|c|c|c|c}\toprule
    \multicolumn{1}{c}{ }  & \multicolumn{6}{c}{Fed-REACT} & \multicolumn{1}{c}{ }  & \multicolumn{1}{c}{ } &  \multicolumn{1}{c}{ } & \multicolumn{1}{c}{ } \\ 
    Algorithm &  C = 1 & C = 3 & C = 9 & C = 25 & C = 40 & C = 50 & LSTM (C=1; Ditto) & Patch-TST (C=1; APFL) & Times-Net (C=1; FedProx) & Causal CNN (C=1; FedProx) \\ \midrule
    RMSE &  24.4 & 23.7 & 13.0 & 8.8 & 5.8 & \textbf{1.3} & 42.0 & 20.1 & 34.3 & 38.2 \\    \bottomrule
  \end{tabular}
}
\caption{Performance on SUMO EV dataset: Fed-REACT with varied values of $C$ alongwith the best results for the supervised learning baselines from the main paper}\label{table:4}
\end{table}

Unlike the other datasets used in the paper, for SUMO we do not a priori know the number of clusters, $C$. This is why we test the performance of our method for various values of C, the total number of clusters, with $C=1$ denoting global averaging of the output layer and $C=50$ denoting complete personalization. The root mean-square error (RMSE) averaged across clients is presented in Table \ref{table:4}, contrasted against the best results from the non-clustered supervised learning baselines. As the data is highly heterogeneous, Fed-REACT outperforms the supervised learning baselines for $C\geq 9$. In fact, the clustering the output layer confers no advantage at all, and the optimal performance is achieved for a completely personalized setting $(C=50)$. However, even when we average the output layer across all the clients  ($C=1$), Fed-REACT still performs competitively against the supervised learning baselines, lagging behind only Patch-TST.

\section{Ablation Study}

\subsection{Ablation study over levels of heterogeneity}
We perform an ablation study on the RTD dataset for stationary cluster distribution, exploring the relationship between heterogeneity, controlled by parameter $\alpha$, and the achieved accuracy averaged across clients. To reiterate, smaller values of $\alpha$ induce greater level of heterogeneity across clusters. We consider a federated learning system with $100$ clients; the number of clients per cluster remains the same as in the previous experiments. The results, presented in Table \ref{table:5}, demonstrate the benefits of the evolutionary strategy that considers past cluster assignments and task model parameters when grouping the clients and aggregating cluster-specific task models.


\begin{table*}[!thbp]
  \centering
  \small
  \resizebox{0.9\linewidth}{!}{%
  \begin{tabular}{c|c|c|c|c|c|c|c}
    \toprule
    $\alpha$ & Fed-REACT w/ A1 & Fed-REACT w/ A2 & SC+MMA & EC+MMA & IFCA & FLSC & FLACC \\ 
    \midrule
    0.10 & 0.888 & \textbf{0.900} & 0.887 & 0.887 & 0.889 & 0.693 & 0.579  \\ 
    \midrule
    0.25 & \textbf{0.872} & 0.871 & 0.868 & 0.868 & 0.872 & 0.761 & 0.620   \\ 
    \midrule
    0.50 & \textbf{0.816} & 0.815 & 0.809 & 0.809 & 0.711 & 0.735 & 0.629  \\ 
    \midrule
    2.00 & \textbf{0.742} & 0.738 & 0.712 & 0.721 & 0.730 & 0.721 & 0.635  \\ 
    \bottomrule
  \end{tabular}}
  \caption{The test accuracy of clustered FL algorithms with varied values of Dirichlet distribution parameter $\alpha$; smaller $\alpha$ indicates higher level of heterogeneity.}
  \label{table:5}
\end{table*}

\subsection{Ablation study over the number of clusters}
For the main experiments on the RTD dataset, we created three clusters. In this section, we perform additional experiments for different number of clusters into which we partition 100 clients. Throughout this section, we set the parameter $\alpha$ to $0.5$. Apart from this exception and the number of clusters, we keep the experimental setting the same as the stationary cluster setting studied above. We explore the following clustering configurations
\begin{itemize}
\item 2 clusters with 50 clients each
\item 4 clusters with 25 clients each
\item 5 clusters with 20 clients each
\item 6 clusters with five clusters having 16 clients each and the sixth cluster containing 20 clients.
\item 7 clusters with six clusters having 14 clients each and the seventh cluster containing 16 clients.

\end{itemize}

\begin{table*}[!thbp]
  \centering
  \small
  \resizebox{0.9\linewidth}{!}{
  \begin{tabular}{c|c|c|c|c|c|c|c}\toprule
    Number of Clusters & Fed - REACT w/ A1 & Fed - REACT w/ A2 &  SC + MMA & EC + MMA & IFCA & FLSC &  FLACC \\ \midrule
    2 & 0.7308 & 0.7312 & 0.7545 & 0.7612 & 0.7316  & 0.6283 &  0.6297   \\ \midrule
    4  & \textbf{0.8146} & 0.8145 & 0.7906 & 0.8058 & 0.7663  &  0.7829 & 0.6585 \\ \midrule
    5  & \textbf{0.7920}  & 0.7914  & 0.7608  & 0.7826  & 0.7550  & 0.7504  &  0.6176  \\ \midrule
    6  & \textbf{0.8445}  & 0.8425  & 0.8120  & 0.8368  & 0.8316 & 0.7755 & 0.6018  \\ \midrule
    7  & \textbf{0.7682}  &  0.7674 & 0.7388  & 0.7550  & 0.7450 & 0.7599 &  0.6604 \\ \bottomrule
  \end{tabular}
  }
  \caption{Ablation results for various cluster configurations with stationary cluster distributions; the clusters cumulatively comprising of $100$ clients were generated using $\alpha=0.5$}
  \label{tab:numcluster_acc}
\end{table*}

\subsection{Ablation study for $\lambda_1$ and $\lambda_2$ in strategy 1 for non-stationary experiments on the RTD dataset}

Due to the lack of space in the main section, we defer to this section the results for various values of the parameters controlling the non-stationary in Scenario 1 for our experiments on the RTD dataset. We present the clustering performance, measured by rand score, in Fig. \ref{fig:nonstat10_75}, \ref{fig:nonstat100_85}, and \ref{fig:nonstat100_95} respectively. The advantage of Fed-REACT in correctly identifying the underlying cluster translates to the accuracies for the 10 client and 100 client setting presented in Tables \ref{tab:cluster_acc_10_non_st} and \ref{tab:cluster_acc_100_non_st}, respectively. 

\begin{table*}[htbp!]
  \centering
  \small
  \resizebox{0.9\linewidth}{!}{
  \begin{tabular}{c|c|c|c|c|c|c|c}\toprule
    Setting & Fed - REACT w/ A1 & Fed - REACT w/ A2 &  SC + MMA & EC + MMA & IFCA & FLSC &  FLACC \\ \midrule
     $\lambda_1=0.95,\lambda_2=0.05$  & \textbf{0.925} & 0.894  & 0.848 & 0.848  & 0.889 & 0.906 &  0.882 \\ \midrule
     $\lambda_1=0.90,\lambda_2=0.10$ & \textbf{0.920} & 0.892  & 0.850 & 0.851 & 0.919 & 0.873 & 0.878   \\ \midrule
     $\lambda_1=0.80,\lambda_2=0.20$ & \textbf{0.919} & 0.866 & 0.823 & 0.825  & 0.800 & 0.864 & 0.884  \\ \midrule
     $\lambda_1=0.75,\lambda_2=0.25$ & \textbf{0.911} & 0.820 & 0.809 & 0.810  & 0.778 & 0.846  & 0.872  \\ \midrule
     $\lambda_1=0.85,\lambda_2=0.50$ & 0.920 & 0.881  & 0.808 & 0.808  & \textbf{0.929} & 0.907  & 0.889   \\ \midrule
     $\lambda_1=0.85,\lambda_2=0.33$ & \textbf{0.908} & 0.905  & 0.799 & 0.800 & 0.875 & 0.892  & 0.872   \\ \midrule
     $\lambda_1=0.75,\lambda_2=0.33$ & \textbf{0.904} & 0.797 & 0.776 & 0.775 & 0.875  & \textbf{0.904} & 0.882  \\ \midrule
  \end{tabular}
  }
  \caption{Results on RTD dataset for $\mathbf{10}$ \textbf{clients}, $100$ rounds with non-stationary cluster distributions (Strategy 1)}
  \label{tab:cluster_acc_10_non_st}
\end{table*}

\begin{table*}[htbp!]
  \centering
  \small
  \resizebox{0.9\linewidth}{!}{
  \begin{tabular}{c|c|c|c|c|c|c|c}\toprule
    Setting & Fed - REACT w/ A1 & Fed - REACT w/ A2 &  SC + MMA & EC + MMA & IFCA & FLSC &  FLACC \\ \midrule
     $\lambda_1=0.95,\lambda_2=0.05$  & 0.785  & \textbf{0.796}  & 0.749  & 0.762   & 0.778 & 0.753 & 0.672   \\ \midrule
     $\lambda_1=0.90,\lambda_2=0.10$ & \textbf{0.787} & 0.725 & 0.733 & 0.744   & 0.695  & 0.731  & 0.679    \\ \midrule
     $\lambda_1=0.80,\lambda_2=0.20$ & \textbf{0.780}  & 0.716 & 0.708 &   0.712 & 0.698 & 0.737 & 0.686   \\ \midrule
     $\lambda_1=0.75,\lambda_2=0.25$ & \textbf{0.777}  & 0.774  & 0.706 & 0.708  & 0.710 & 0.772 & 0.703   \\ \midrule
     $\lambda_1=0.85,\lambda_2=0.50$ & 0.786  & \textbf{0.801} & 0.701  & 0.706  & 0.774 & 0.739 & 0.698    \\ \midrule
     $\lambda_1=0.85,\lambda_2=0.33$ & \textbf{0.782} & 0.704 & 0.683 & 0.690   & 0.727 & 0.724 & 0.723    \\ \bottomrule
  \end{tabular}
  }
  \caption{Results on RTD dataset for $\mathbf{100}$ \textbf{clients}, $200$ rounds with non-stationary cluster distributions (Strategy 1)}
  \label{tab:cluster_acc_100_non_st}
\end{table*}

\begin{figure*}[!htbp]
\centering
  \begin{minipage}[t]{.32\linewidth}
    \includegraphics[width=\linewidth]{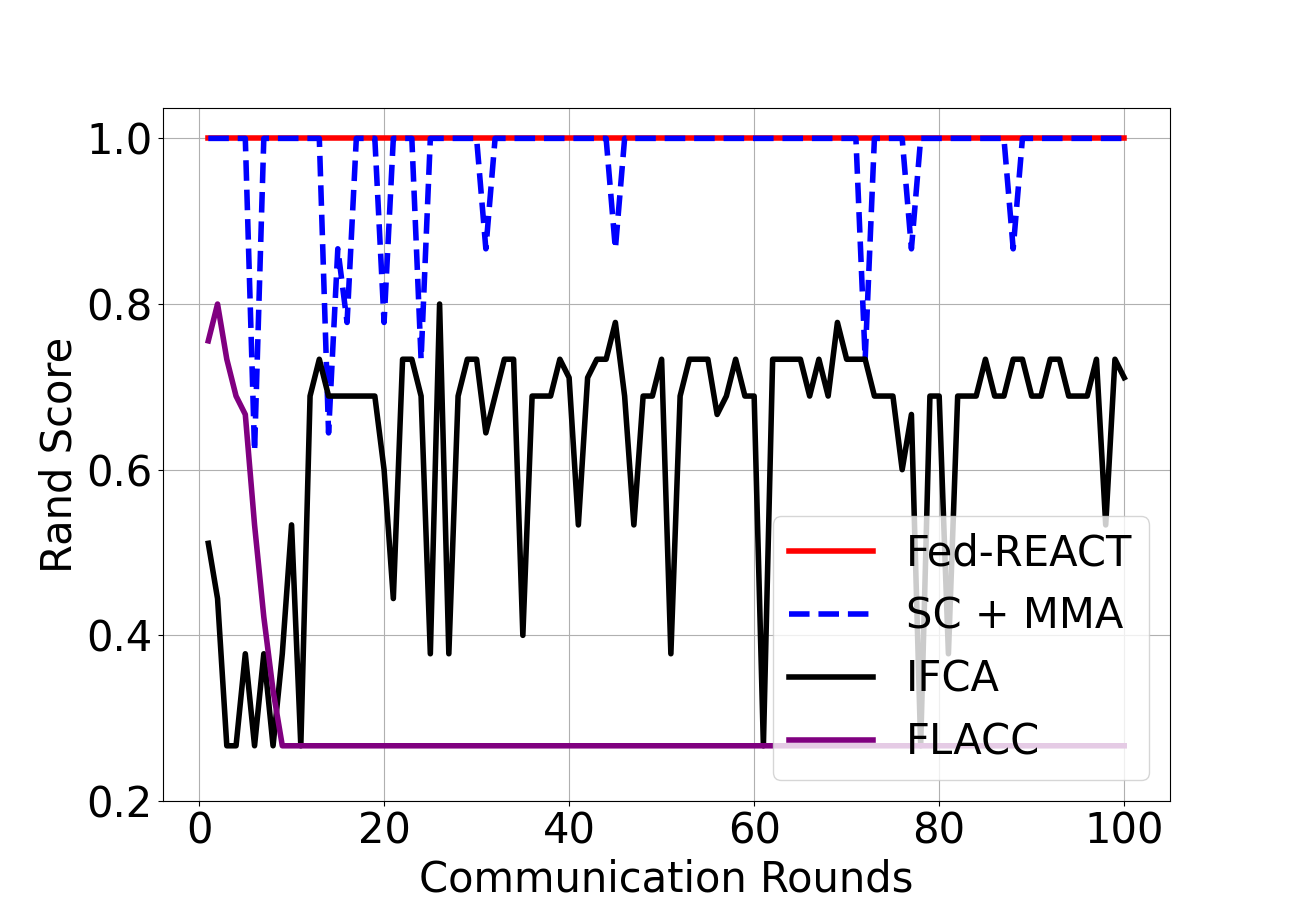}%
    \caption{\small The Rand score of Fed-REACT vs. baseline methods on RTD for Non Stationary Setting (Strategy 1), $10$ clients,  $|\mathcal{M}^k_t|=64$ training samples, $T_{task}=100$., $\lambda_1=0.75, \lambda_2=0.25$}%
    \label{fig:nonstat10_75}
  \end{minipage}\hfil
  \begin{minipage}[t]{.32\linewidth}
    \includegraphics[width=\linewidth]{Non_Stat_0_85_100_clients.png}%
    \caption{\small The Rand score of Fed-REACT vs. baseline methods on RTD for Non Stationary Setting (Strategy 1), $100$ clients,  $|\mathcal{M}^k_t|=64$ training samples, $T_{task}=200$., $\lambda_1=0.85, \lambda_2=0.15$}%
    \label{fig:nonstat100_85}
  \end{minipage}\hfil
  \begin{minipage}[t]{.32\linewidth}
    \includegraphics[width=\linewidth]{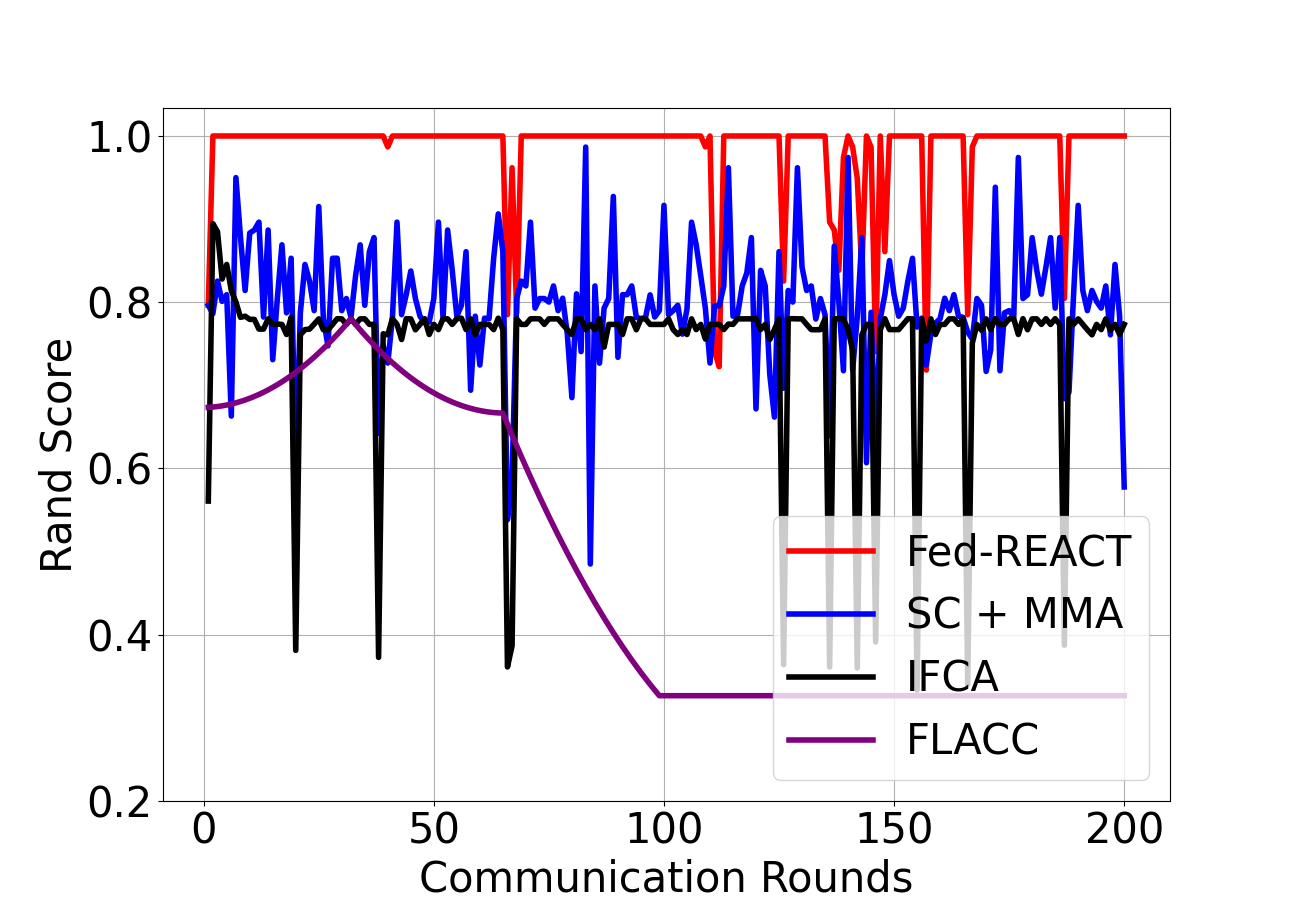}%
    \caption{\small The Rand score of Fed-REACT vs. baseline methods on RTD for Non Stationary Setting (Strategy 1), $100$ clients,  $|\mathcal{M}^k_t|=64$ training samples, $T_{task}=200$., $\lambda_1=0.95, \lambda_2=0.05$}%
    \label{fig:nonstat100_95}
  \end{minipage}%
\end{figure*}





\section{Experimental results on time-smoothed gradient descent}

The time-smoothed gradient descent algorithm DTSSGD, proposed by (\cite{aydore2019dynamic}), presents a regret framework for non-convex models that deals with the concept drift associated with a dynamic environment. We compare our results with those obtained by training the encoder using DTSSGD. The experiments are conducted on the RTD dataset with ten clients partitioned into 3 clusters created using Dirichlet sampling ($\alpha=0.1$). As before, the encoder was trained for $10$ rounds but with the optimizer set to the one proposed in \cite{aydore2019dynamic}. Training of the output layer consists of a single round involving all the labeled samples available at a client. We vary the parameter $\gamma$ (used to control forgetting) and the smoothing window size $w$. The results are presented in Table \ref{tab:alpha_window}.

\begin{table}[ht]
\centering
\begin{tabular}{c|c|c|c|c}
\toprule
\textbf{$\gamma$} & \textbf{$w = 1$} & \textbf{$w = 3$} & \textbf{$w = 5$} & \textbf{$w = 7$} \\ \hline
0.7           & 0.988             & 0.984        & 0.988            & 0.986                 \\ \hline
0.8           & 0.988           & 0.990            & 0.982            & 0.985            \\ \hline
0.9           & 0.988             & 0.980            & 0.990            & 0.990            \\ \bottomrule
\end{tabular}
\caption{Fed-REACT's results using the optimizer from \cite{aydore2019dynamic}}
\label{tab:alpha_window}
\end{table}
The results suggest that increasing $w$ does not lead to significant performance gain; therefore, in our experiments we set $w = 1$.

\section{Complexity analysis}
\subsection{Comparison of representation learning schemes against the baselines}
For $L$ layers, a kernel size of $\nu$, a stride of $1$, and 
a timeseries length of $D$, the inference complexity of a 1-D convolutional neural network is $\mathcal{O}(\nu L D)$. Since time dilation simply spaces out the successive neurons sharing the same kernel, the computational complexity remains linear in terms of the sequence length, $\mathcal{O}(\nu L D)$. Attention based PatchTST, on the other hand, scales quadratically with the sequence length $\mathcal{O}(L D^2)$ which makes it inefficient for longer sequences. 

The communication efficiency of the backbone is proportional to the number of parameters for each model, which are presented in Table \ref{tab:model_sizes}. Note that the number of parameters does not reflect the size of the memory required to train the models. The number of layers, kernel sizes, etc. for each of these models were chosen so that NVIDIA A-100 GPUs could be maximally utilized in terms of memory for a federated learning setup with $50$ clients trained using fedml library. 
\begin{table*}[h]
  \centering
  \small
  {
  \begin{tabular}{c|c|c|c|c}\toprule
     & Causal CNN & LSTM &  Patch TST & Times Net \\ \midrule
    Number of Parameters (K) & 156.59 & 398.35  & 131.05 &  89.13  \\ \midrule
    FLOPs (M) & 31.94 & 79.68 & 2.66 & 17.6 \\ \bottomrule
  \end{tabular}
  }
  \caption{Model sizes for various schemes used in our experiments}
  \label{tab:model_sizes}
\end{table*}

Please note that the training in Phase 1 of Fed-REACT is simple Federated Averaging of the encoder at the server end, which scales linearly with the number of clients and model size. Since Fedprox involves a slight reparameterization of Federated Averaging by modifying the loss function at the client side to include a proximal term, the complexity of the two methods are the same. Ditto also modifies the local loss function by regularizing deviation of the local models from the global models. If the local device runs simple SGD on the modified local loss, the complexity of Ditto remains the same as that of Federated Averaging. In APFL, each client maintains three models: global model, local model, and mixed personalized model that is a combination of local and global models. Although this introduces obvious overhead at client side in terms of memory and computation, the overall complexity remains the same as for the vanilla federated averaging.  

\subsection{Evolutionary clustering vs other clustered Federated Learning methods}
Compared to regular snapshot clustering, $\Gamma$ iterations of AFFECT algorithm incur an obvious computational overhead. However, in our experiments we have observed that $\Gamma \ll K$, and the algorithm converges in less than $5$ iterations. Given the clustering solution $\mathcal{S}_t^c$ obtained in the current iteration, $\hat{\mathbb{E}}[[W_t]_{i, j}]$ and $\hat{Var}([W_t]_{i, j})$ are computed in $\mathcal{O}(K^2)$ operations. Likewise, the computation of the forgetting factor given the mean and variances requires $\mathcal{O}(K^2)$ operations. Suppose that the computational complexity of the clustering algorithm employed is $T(K,C)$ operations, then the computational complexity of AFFECT turns out to be $\Gamma T(K,C)+\mathcal{O}(\Gamma K^2)$. Given that the k-means clustering is an NP hard problem, and the time-complexity of Agglomerative Hierarchical Clustering is $\mathcal{O}(K^2 log K)$, the incurred overhead is negligible provided that $\Gamma \ll K$ is satisfied as in our experiments. 

The comparison between evolutionary clustering and the algorithms that cluster at the client side (e.g., IFCA and FLSC) is not straightforward. Instead of traditional clustering based on the weights of the output layer at the server side, the server in the latter schemes shares the weights of all $C$ clusters with each of the client. The clients then evaluate the performance of each of these cluster models and pick the cluster with the lowest loss. Hence, the clustering for these schemes depends on both the size of the task model and the number of clusters. It is clear that both IFCA and FLSC are worse off than Fed-REACT in terms of communication efficiency. For Fed-REACT, the server incurs communication cost of $\mathcal{O}(K |\mathbf{\theta}_{task, t}|)$, while for the former two it incurs a communication cost of $\mathcal{O}(C K |\mathbf{\theta}_{task, t}|)$.

\end{document}